\documentclass[a4paper]{amsart} 

\usepackage[english]{babel}
\usepackage[utf8x]{inputenc}
\usepackage[T1]{fontenc}
\usepackage{dsfont}

\usepackage{amsmath}
\usepackage{amsfonts}
\usepackage{graphicx}
\usepackage[colorinlistoftodos]{todonotes}
\usepackage[colorlinks=true, allcolors=blue]{hyperref}

\newcommand{\Set}[1]{\mathbb{#1}}
\newcommand{\R}{\mathbb{R}}

\newcommand{\N}{\mathbb{N}}
\newcommand{\I}{\Set{I}}
\newcommand{\J}{\Set{J}}
\newcommand{\calC}{\mathcal{C}}
\newcommand{\calR}{\mathcal{R}}
\newcommand{\calD}{\mathcal{D}}
\newcommand{\diver}{\nabla\cdot}
\newcommand{\grad}{\nabla}
\newcommand{\dt}{\hspace{0.2em} \mathrm{d}t}
\newcommand{\id}{\mathsf{Id}}
\newcommand{\lapl}{\Delta}
\newcommand{\dontshow}[1]{ }

\numberwithin{equation}{section}

\newcommand{\rmd}{\mathrm{d}}
\newcommand{\rme}{\mathrm{e}}

\newcommand{\dbydab}[3]{\frac{\partial^{2} {#1}}{\partial {#2} \partial {#3}}}

\newtheorem{theorem}{Theorem}[section]
\newtheorem{lemma}[theorem]{Lemma}

\newtheorem{definition}[theorem]{Definition}
\newtheorem {remark}[theorem]{Remark}

\newtheorem {conjecture}[theorem]{Conjecture}

\title[Networks for Nonlinear Diffusion]{Networks for Nonlinear Diffusion Problems in Imaging}

\author{S. Arridge and A. Hauptmann}

\date{\today}
\thanks{S. Arridge and A. Hauptmann are with the Department of Computer Science; University College London, London, United Kingdom}

\begin{document}
\maketitle

\begin{abstract}
A multitude of imaging and vision tasks have seen recently a 
major transformation by deep learning methods and in particular 
by the application of convolutional neural networks. These 
methods achieve impressive results, even for applications where 
it is not apparent that convolutions are suited to capture the 
underlying physics. 

In this work we develop a network architecture based on nonlinear diffusion processes, named \emph{DiffNet}. By design, we obtain a nonlinear network 
architecture that is well suited for diffusion related problems 
in imaging. Furthermore, the performed updates are explicit, by which we obtain better interpretability and generalisability 
compared to classical convolutional neural network architectures. The performance of DiffNet tested on the inverse problem of 
nonlinear diffusion with the Perona-Malik filter on the STL-10 
image dataset. We obtain competitive results to the established 
U-Net architecture, with a fraction of parameters and necessary 
training data.
\end{abstract}

\section{Introduction}
We are currently undergoing a paradigm shift in imaging and 
vision tasks from classical analytic to learning and data based 
methods. In particular by deep learning  and the application of 
convolutional neural networks (CNN). Whereas highly superior 
results are obtained, interpretability and analysis of the 
involved processes is a challenging and ongoing task \cite{hofer2017,kim2018}.

In a general setting, Deep Learning proposes to develop a non-linear mapping $A_{\Theta} : X\to Y $ between elements of two spaces $X$, $Y$ (which may be the same) parametrised by a finite set of parameters $\Theta$, which need to be learned:
\begin{equation}
g = A_{\Theta} f
\label{eq:generalDL1}
\end{equation}
This learning based approach is in marked contrast to classical methods with a physical interpretation of the process \eqref{eq:generalDL1} for both computational modelling of data given a physical model (which we call a \emph{forward problem}), and the estimation of parameters of a physical model from measured, usually noisy, data (which we call an \emph{inverse problem}). Several recent papers have discussed the application of Deep Learning in both forward \cite{khoo2017,raissi2017,Sirignano2017,Tompson2017,Weinen2017,Wu2018} and inverse \cite{Jin2017,Kang2017,MeinhardtICCV2017,ye2018,Zhu2018}
problems. Several questions become essential when applying learned models to both forward and inverse problems including:
\begin{itemize}
    \item how and to what extent can learned models replace physical models?
    \item how do learned models depend on training protocols and how well do they generalise?
    \item what are appropriate architectures for the learned models, what is the size of the parameter set $\Theta$ that needs to be learned, and how can these be interpreted?
\end{itemize}
In this paper we aim to answer some of these questions, by taking a few steps back and looking at an analytic motivation for network architectures. Here we consider in particular mappings between images $u$ in dimension $d$, i.e. $X = Y = L^p(\Omega \subset \R^d)$, for which there exist several widely 
used linear and nonlinear mappings $A$ defined by differential 
and/or integral operators. For example a general linear integral 
transform such as
\begin{equation}
u^{\rm obs}(x) = (A u^{\rm true})(x) = \int_{\Omega} K(x,y) u^{\rm true}(y) \rmd y    
\label{eq:LinearIntegralTransform}
\end{equation}
includes stationary convolution as a special case if the kernel is translation invariant, i.e. $K(x,y) \equiv K(x-y)$. If the kernel depends on the image, i.e. $K(x,y) \equiv K(x,y; u)$ then
\eqref{eq:LinearIntegralTransform} becomes nonlinear. Alternatively, the forward model may be modelled as the end-point of an evolution process
which becomes nonlinear if the kernel depends on the image state
$K(x,y,t) \equiv K(x,y;u(t))$; see eq.(\ref{eq:AnisoGreensConvSoln}) below for a specific example. 


Furthermore, a widely studied class of image mappings is
characterised by defining the evolution through a partial 
differential equation (PDE) where flow is generated by the local 
structure of $u$ \cite{kimmel2003,Sapiro2006,Weickert1998}, such that
\begin{equation}\label{eqn:evolutionIntro}
u_t = F\left(u,\nabla u, \dbydab{u}{x_i}{x_j},\ldots \right) \, .
\end{equation}
These problems in general do not admit an explicit integral transform representation and are solved instead by numerical techniques, but since they depend on a physical model, described by the underlying PDE, they can be analysed and understood 
thoroughly \cite{weickert1998paper}. This also includes a statistical interpretation \cite{helin2010hierarchical} as hyperpriors in a Bayesian setting \cite{calvetti2008a}. Furthermore, models like \eqref{eqn:evolutionIntro} can be used to develop priors for nonlinear inverse problems, such as optical tomography \cite{douri2005a,hannukainen2015} and electrical impedance tomography \cite{Hamilton2014}.

Motivated by the success of these analytic methods to imaging 
problems in the past, we propose to combine physical models with 
data driven methods to formulate network architectures for 
solving both forward and inverse problems that take the 
underlying physics into account. We limit ourselves to the case 
where the physical model is of diffusion type, although more 
general models could be considered in the future. The leading 
incentive is given by the observation that the underlying 
processes in a neural network do not need to be limited to 
convolutions.

Similar ideas of combining partial differential equations with 
deep learning have been considered earlier. For instance learning coefficients of a PDE via optimal control \cite{liu2010}, as well as deriving CNN architectures motivated by diffusion processes \cite{chen2015}, deriving stable architectures by drawing 
connections to ordinary differential equations \cite{haber2017} 
and constraining CNNs \cite{ruthotto2018} by the interpretation 
as a partial differential equation. Another interpretation of our approach can be seen as introducing the imaging model into the 
network architecture; such approaches have led to a major 
improvement in reconstruction quality for tomographic problems \cite{Adler2017,hammernik2018,Hauptmann2018TMI}.

This paper is structured as following. In section \ref{sec:Diff4imaging} we review some theoretical aspects of 
diffusion processes for imaging and the inversion based on theory of partial differential equations and differential operators. We 
formulate the underlying conjecture for our network architecture, that the diffusion process can be inverted by a set of local 
non-stationary filters. In the following we introduce the notion 
of \emph{continuum networks} in section \ref{sec:contNet} and 
formally define the underlying \emph{layer operator} needed to formulate network architectures in a continuum setting. We draw connections to the established convolutional neural networks in our continuum setting. We then proceed to define the proposed layer operators 
for diffusion networks in section \ref{sect:DiffNet} and derive 
an implementable architecture by discretising the involved 
differential operator. In particular we derive a network 
architecture that is capable of reproducing inverse filtering 
with regularisation for the inversion of nonlinear diffusion 
processes. We examine the reconstruction quality of the proposed 
\emph{DiffNet} in the following section \ref{sec:compExp} for an 
illustrative example of deconvolution and the challenging 
inverse problem of inverting nonlinear diffusion with the 
Perona-Malik filter. We achieve results that are competitive to 
popular CNN architectures with a fraction of the amount of  
parameters and training data. Furthermore, all computed 
components that are 
involved in the update process are interpretable and can be 
analysed empirically. In section \ref{sec:discussion} we examine 
the generalisablity of the proposed network with respect to 
necessary training data. Additionally, we empirically analyse the
obtained filters and test our underlying conjecture. Section \ref{sec:conclusions} presents some conclusions and further ideas.



\section{Diffusion and flow processes for imaging}\label{sec:Diff4imaging}
In the following we want to explore the possibility to include a 
nonlinear process as an underlying model for network 
architectures.
Specifically, the motivation for this study is given by diffusion processes that have been widely used in imaging and vision. 
Let us consider here the general diffusion process in $\R^d$;
then on a fixed time interval with some diffusivity $\gamma$, to be defined later, we have
\begin{eqnarray}
\left\{ 
\begin{array}{rll}
\partial_t u  \hspace{-0.5em} &=\ \diver(\gamma \nabla u) & \text{in }\R^d\times (0,T] \vspace{0.1cm} \\
u(x,0) \hspace{-0.5em} &=\ u_0(x) & \text{in } \R^d.
\end{array}
\right.\label{eqn:diffusionEquation}
\end{eqnarray}
\begin{remark}
When considering bounded domains $\Omega \subset \R^d$ we will augment \eqref{eqn:diffusionEquation} with boundary conditions on $\partial \Omega$. We return to this point in section~\ref{sect:DiffNet}.
\end{remark}
In the following we denote the spatial derivative by
\begin{equation}\label{eqn:spatDiffOp}
\left(\mathcal{L}(\gamma)u\right)(x,t):= \diver(\gamma \nabla u(x,t)).
\end{equation}
Let us first consider the \emph{isotropic diffusion} case, then the differential operator becomes the spatial Laplacian $\mathcal{L}(\gamma = 1)= \lapl$.
In this case, the solution of \eqref{eqn:diffusionEquation} at time $T$ is given by convolution with a \emph{Green's function} 
\begin{equation}
u_T(x) = G_{\sqrt{2T}}(x)\ast u_0(x)
\label{eq:IsoGreensConvSoln}
\end{equation}
where $ G_{\sqrt{2T}} = \frac{1}{(4\pi T)^{d/2}}\exp\left[-\frac{x^2}{4T}\right]$ in dimension $d$ and we recall that the convolution of two functions $f,g\in L^1(\R^d)$ is defined by
\begin{equation}
(g\ast f)(x) = \int_{\R^d} g(x-y)f(y) \rmd y.
\label{eq:conv2func}
\end{equation}
\begin{definition} The  \emph{Green's operator} is defined with the Green's function as its kernel
\begin{equation}
\mathcal{G}_{\frac{1}{2}\sigma^2} u := \int_{\mathbb{R}^d} \frac{u(x)}{(2\pi \sigma^2)^{d/2}}\exp\left[-\frac{|x-y|^2}{2\sigma^2}\right]\rmd y    
\end{equation}
by which we have
\[
u_T(x) = \mathcal{G}_{\rm T} u_0 (x)
\]
\end{definition}

In the general case for an \emph{anisotropic diffusion flow} (ADF) we are interested in a scalar \emph{diffusivity} $\gamma \in [0,1]$ that depends on $u$ itself, i.e. 
\begin{equation}
\partial_t u= \nabla\cdot\left(\gamma(u) \nabla u\right)
\label{eq:AnisoDiff1}
\end{equation}

This is now an example of 
a non-linear evolution
\begin{equation}
u_T(x) = \mathcal{K}_T u_0 = \int_0^T\int_{\mathbb{R}^d} K^{\rm ADF}(x,y,u(y,t) ) u_0(y) \rmd y \rmd t
\label{eq:AnisoGreensConvSoln}
\end{equation}
where $K^{\rm ADF}(x,y,u(x,t))$ is now a non-stationary, non-linear and time-dependent kernel. 
In general there is no explicit expression for $K^{\rm ADF}$ and numerical methods are required for the solution of \eqref{eq:AnisoDiff1}.

\begin{remark}
Not considered here, but a possible extension to 
 \eqref{eq:AnisoDiff1} is where $\gamma$ is a tensor, which, for $d=2$ takes the form
\begin{equation*}
\partial_t u = 
\nabla\cdot\begin{pmatrix}\gamma_{11} & \gamma_{12} \\ \gamma_{12} & \gamma_{22} \end{pmatrix} \nabla u
\label{eq:TensorAnisoDiff1}
\end{equation*}
Furthermore, extensions 
exist for the case where $u$ is vector or tensor-valued. We do not consider these cases here, see \cite{Weickert1998} for an overview.
\end{remark}

\subsection{Forward Solvers}
First of all, let us establish a process between two states of the function $u$. Integrating over time from $t=t_0$ to $t = t_1=t_0+\delta t$ yields
\[
\int_{t_0}^{t_1} \partial_t u(x,t) \dt = \int_{t_0}^{t_1}  \left(\mathcal{L}(\gamma)u\right)(x,t)  \dt.
\]
Note, that the left hand side can be expressed as $\int_{t_0}^{t_1} \partial_t u(x,t) \dt = u(x,t_1)-u(x,t_0)$ and we denote the right hand side by an integral operator $\mathcal{A}_{\delta t}(\gamma)$, such that
\begin{equation}\label{eqn:diffIntegral}
(\mathcal{A}_{\delta t}(\gamma)u)(x,t_0) := \int_{t_0}^{t_1=t_0+\delta t} \left(\mathcal{L}(\gamma)u\right)(x,t)  \dt.
\end{equation}
In the following we denote the solution of \eqref{eqn:diffusionEquation} at time instances $t_n$ as $u^{(n)}=u(x,t_n)$. Then we can establish a relation between two time instances of $u$ by
\begin{equation}\label{eqn:diffUpdate}
u^{(n+1)} = (Id + \mathcal{A}_{\delta t}(\gamma) )u^{(n)}=   u^{(n)} + \int_{t_n}^{t_{n+1}} \mathcal{L}(\gamma)u(x,t) \dt,
\end{equation}
where $Id$ denotes the identity and  $t_{n+1}=t_n+\delta t$.

Since we cannot compute $u^{(n+1)}$ by \eqref{eqn:diffUpdate} without the explicit knowledge of the (possibly time-dependent) diffusivity $\gamma$, it is helpful to rather consider a fixed diffusivity at each time instance $\gamma^{(n)}=\gamma(x,t=t_n)$, or $\gamma^{(n)}=\gamma(u(x,t=t_n))$ in the nonlinear case; then by using the differential operator \eqref{eqn:spatDiffOp} we have an approximation of \eqref{eqn:diffIntegral} by
$$
\delta t \mathcal{L}(\gamma^{(n)}) u^{(n)}= \delta t (\diver \gamma^{(n)} \nabla u^{(n)})  \approx \mathcal{A}_{\delta t}(\gamma) u^{(n)}.
$$
We can now solve \eqref{eq:AnisoDiff1} approximately by iterating for time steps $\delta t$ using either an \emph{explicit} scheme 
\begin{equation}
\calD_{\delta t}^{{\rm Expl}}(\gamma^{(n)})u^{(n)} = \left(Id + \delta t {\mathcal{L}}(\gamma^{(n)}) \right) u^{(n)}\,,
\label{eq:Explicit1}
\end{equation}
or an \emph{implicit} scheme
\begin{equation}
\calD_{\delta t}^{{\rm Impl}}(\gamma^{(n)})u^{(n)} = \left(Id - \delta t {\mathcal{L}}(\gamma^{(n)}) \right)^{-1} u^{(n)}\,,
\label{eq:Implicit1}
\end{equation}
Whereas \eqref{eq:Explicit1} is stable only if CFL conditions are satisfied and \eqref{eq:Implicit1} is unconditionally stable, they are both only accurate for sufficiently small steps $\delta t$. In fact, by the Neumann series, the schemes are equivalent in the limit
\begin{equation}
\underset{\delta t \rightarrow 0}{\lim}\left(Id - \delta t {\mathcal{L}}(\gamma) \right)^{-1} = Id + \delta t {\mathcal{L}}(\gamma)  + \mathcal{O}((\delta t)^2)\, 
\end{equation}
and coincide with the integral formulation of \eqref{eqn:diffUpdate}.

It's also useful to look at the Green's functions solutions. 
\begin{lemma}\label{lemma:IteratedGreensFunction}
Consider the isotropic case with $\gamma\equiv 1$. Then we may write, with time steps $\delta t = T/N$, 
\begin{equation}
\mathcal{G}_Tu_0 = G_{\sqrt{2 T}} \ast u_0 = \underbrace{G_{\sqrt{2 \delta t}} \ast \ldots \ast G_{\sqrt{2 \delta t}}}_{\text{$N$-times}} \ast u_0 = 
\underbrace{\mathcal{G}_{\delta t}\circ\ldots\mathcal{G}_{\delta t}}_{\text{$N$-times} } \circ u_0
\label{eq:IteratedGreensFunction}
\end{equation}
\end{lemma}
\begin{proof}
Take the Fourier Transform\footnote{Note the definition of Fourier Transform is chosen to give the correct normalisation so that $\hat{u}(k)|_{k=0} = \int_{{\mathrm{R}}^d} u(x) \rmd^n x $} 
 $\hat{G}_{\sigma}(k) =  {\mathcal{F}}_{x\rightarrow k} G_{\sigma}(x) =  \rme^{-\frac{\sigma^2 k^2}{2}}$ and use the convolution theorem to give
\begin{eqnarray*}
\hat{G}_{\sqrt{2 T}}(k)  \hat{u}_0(k)  &=&\left( \Pi_{n=1}^N\hat{G}_{\sqrt{2 \delta t}}(k) \right) \hat{u}_0(k) \\
\rme^{-{k^2T} }  \hat{u}_0(k) &= & \left(\rme^{-{k^2 \delta t} }  \right)^N \hat{u}_0(k)   = \rme^{-{k^2 N\delta t} }  \hat{u}_0(k)
\end{eqnarray*}
\end{proof}

\noindent
Let us also note that in Fourier domain, by Taylor series expansion we have  
\[
\exp(-k^2 \delta t) \rightarrow 1 - k^2\delta t  + \frac{1}{2} k^4 (\delta t)^2 - \ldots \, ,
\]
and therefore in the spatial domain the finite difference step and the Gaussian convolution step are the same
\begin{equation}
\lim_{\delta t \rightarrow 0}\left(G_{\sqrt{2 \delta t} }\ast u_0 \right) = \left(Id + \delta t \lapl + \mathcal{O}((\delta t)^2)\right)\ast u_0  = \lim_{\delta t \rightarrow 0}\left(Id - \delta  t \lapl \right)^{-1}u_0 \,.
\end{equation}

\subsection{Inverse Filtering}
Let us now consider the inverse problem of reversing the diffusion process. That is we have $u_T$ and aim to recover the initial condition $u_0$. This is a typical ill-posed problem as we discuss in the following.

\subsubsection{Isotropic case  $\gamma \equiv 1$}
As the forward problem is represented as convolution in the spatial domain, the inverse mapping $u_T \mapsto u_0$ is a (stationary) \emph{deconvolution}. We remind that $\hat{u}_T=\rme^{-{k^2 T} }  \hat{u}_0(k)$, then the inversion is formally given by division in the Fourier domain as
\begin{equation}
u_0(x) = {\mathcal{F}}^{-1}_{k\rightarrow x}\left[\hat{u}_T(k) \rme^{k^2 T} \right]\,.
\label{eq:FourierDeconv}
\end{equation}
However we note: 
\begin{itemize}
    \item[i.)] The factor $\rme^{k^2 T} $ is unbounded and hence the equivalent convolution kernel in the spatial domain does not exist.
    \item[ii.)] (\ref{eq:FourierDeconv}) is unstable in the presence of even a small amount of additive noise and hence it has to be \emph{regularised} in practice.
\end{itemize}

Nevertheless, let's consider formally with $\rme^{k^2  T} = \left(\rme^{k^2 \delta t} \right)^N$ that by Taylor series we get 
\[
{\mathcal{F}}^{-1}_{k\rightarrow x}\rme^{k^2 \delta t} \approx {\mathcal{F}}^{-1}_{k\rightarrow x}\left[ 1 +  k^2\delta t + (k^2\delta t)^2 + \dots   \right]
= 1 - \delta t  \lapl +  \mathcal{O}((\delta t)^2)\,.
\] 
Motivated by this we define an operator for the inversion process
\begin{equation}
\mathcal{E}^{\rm iso}_{\delta t} u := \left(Id - \delta t \lapl \right)u  \, \simeq \mathcal{G}^{-1}_{\delta t} u.
\end{equation}
Clearly $\mathcal{E}_{\delta t}^{\rm iso}$ coincides with the inverse of the implicit update in \eqref{eq:Implicit1}, and 
\begin{equation}
\tilde{u}_{0} =  \underbrace{{\mathcal{E}}^{\rm iso}_{\delta t} \circ \ldots \circ {\mathcal{E}}^{\rm iso}_{ \delta t}}_{N-\textnormal{times}} \circ\, u_T
\label{eq:IteratedInvGreensFunction}
\end{equation}
is an estimate for the deconvolution problem which (in the absence of  noise) is correct in the limit
\begin{equation}
\lim_{\delta t \rightarrow 0} \tilde{u}_{0} \rightarrow u_0.
\end{equation}

\subsubsection{Anisotropic case}
In this case the diffused function is given by 
\eqref{eq:AnisoGreensConvSoln}. 
Following 
Lemma~\ref{lemma:IteratedGreensFunction}
we may put 
\begin{equation}
u_T = \mathcal{K}_T u_0 \quad \simeq \quad \tilde{u}_T :=  \calD_{\delta t}^{{\rm Expl}}(\gamma^{(N-1)})  \circ \ldots\circ \calD_{\delta t}^{{\rm Expl}}(\gamma^{(0)})   u_0 
\label{eq:AnisoIteratedGreens}
\end{equation}
and we also have
\begin{equation}
\lim_{\delta t \rightarrow 0}\tilde{u}_{T} \rightarrow u_T.
\end{equation}
\noindent

\begin{conjecture}
There exists a set of local (non-stationary) filters $\mathcal{E}_{\delta t} (\zeta)$ where
\begin{equation}\label{eqn:nonstatFilter}
\mathcal{E}_{\delta t} (\zeta) u = u - \delta t\int_{{\mathbb{R}}^d} \zeta(x,y) u(y) \rmd y
\end{equation}
and where $\zeta(x,y)$ has only local support and such that
\begin{equation}
u_0 = \mathcal{K}_T^{-1}u_T \quad \simeq \quad  \tilde{u}_0:= \mathcal{E}_{\delta t} ( \zeta^{(N-1)})\circ 
\ldots\circ \mathcal{E}_{\delta t} ( \zeta^{(0)}) u_T \,.
\label{eq:AnisoIteratedInvGreens}
\end{equation}
\end{conjecture}

\begin{remark}[Unsharp Masking]
We recall that a simple method for "deconvolution" is called \emph{Unsharp Masking} which is usually considered as
\begin{equation*}
u^{\rm obs} \mapsto \tilde{u} = u + \epsilon(u^{\rm obs}-G_{\sigma}\ast u^{\rm obs} )
\end{equation*}
for some blur value $\sigma$ and sufficiently small $\epsilon$. By similar methods as above, we find 
\begin{eqnarray*}
\hat {\tilde{u}  }(k) &=& \hat{u}^{\rm obs}(k) +\epsilon\left(Id - \rme^{-\frac{\sigma^2k^2}{2}} \right)\hat{u}^{\rm obs}(k) \simeq \left(Id  + {\frac{\epsilon \sigma^2k^2}{2}} \right)\hat{u}^{\rm obs}(k)\\
\Rightarrow \tilde{u}  &\simeq&  \left(Id  - \frac{\epsilon \sigma^2}{2}  \lapl  \right)u^{\rm obs}(x)\, .
\end{eqnarray*}
We may choose to interpret the operators $ {\mathcal{E}}_{\delta t}(\zeta) $ as a kind of "non-stationary unsharp masking".
\end{remark}

\subsection{Discretisation}\label{sec:PDEdiscretisation}

We introduce the definition of a sparse matrix operator representing local non-stationary convolution
\begin{definition}
$\mathsf{W}$ is called a \emph{Sparse Sub-Diagonal} (SSD) matrix if its non-zero entries are all on sub-diagonals corresponding to the local neighbourhood 
of pixels on its diagonal. 
\end{definition}
Furthermore we are going to consider that a class of SSD matrices $\mathsf{W}(\zeta)$ with learned parameters $\zeta$ can be decomposed as $\mathsf{W}(\zeta) = \mathsf{S}(\zeta) + \mathsf{L}(\zeta)$ where $\mathsf{S}$ is \emph{smoothing} and $\mathsf{L}(\zeta)$ is \emph{zero-mean}; i.e. $\mathsf{L}(\zeta)$ has one zero eigenvalue such that its application to a constant image gives a zero valued image 
\[
\mathsf{L}(\zeta) \mathds{1} = 0
\]

In the following we restrict ourselves to the typical 4-connected neighbourhood of pixels in dimension $d=2$.
For the numerical implementation we have the Laplacian stencil 
$$\lapl \quad \rightarrow \quad \mathsf{L}_{\lapl} = \begin{pmatrix} &1&\\1&-4&1\\ &1& \end{pmatrix} $$ 
from which we have that $\mathsf{L}_{\lapl}$ is zero-mean. 
Similarly we will have for the numerical approximation of  ${\mathcal{E}}^{\rm iso}_{\delta t}$ the matrix operator
$$ Id - \delta t \lapl \quad \rightarrow \quad {\mathsf{E}}^{\rm iso}_{\delta t} = 
\begin{pmatrix} & 0 &\\ 0 & 1  & 0 \\ & 0 & \end{pmatrix} -
\begin{pmatrix} &\delta t&\\\delta t&-4\delta t&\delta t\\ &\delta t& \end{pmatrix}. $$

Further we conjecture that in the numerical setting  $\mathcal{E}_{\delta t} (\zeta)$ is approximated by the sum of identity plus a SSD matrix operator 
as 
\begin{equation}\label{eqn:nonStatStencil}
 \mathcal{E}_{\delta t} (\zeta)  \,\sim\, {Id} - \delta t
 \mathcal{L}(\zeta) \rightarrow \mathsf{Id} -
 \mathsf{L}_{\delta t}(\zeta) \,\sim\,
\mathsf{E}_{\delta t} (\zeta) = \begin{pmatrix} & 0 &\\ 0 & 1  & 0 \\ & 0 & \end{pmatrix} - \delta t
\begin{pmatrix} &\zeta_1 &\\ \zeta_2&-\sum_i \zeta_i &\zeta_4 \\ & \zeta_3 & \end{pmatrix}. 
\end{equation}
In the presence of noise, the inverse operator $\mathcal{E}_{\delta t} (\zeta)$ can be explicitly regularised by addition of a smoothing operation 
\begin{equation}
\tilde{\mathcal{E}}_{\delta t} (\zeta) = \mathcal{E}_{\delta t} (\zeta) +\alpha \mathcal{S} \rightarrow \mathsf{Id} - 
 \mathsf{L}_{\delta t}(\zeta) + \alpha \mathsf{S} =: \mathsf{Id} - \mathsf{W}_{\delta t}(\zeta)
\end{equation}

Whereas in classical approaches to inverse filtering the
regularisation operator would be defined \emph{a priori}, 
the approach in this paper is to learn the operator $\mathsf{W}$
and interpret it as the sum of a differentiating operator
$\mathsf{L}$ and a (learned) regulariser $\mathsf{S}$. This is 
discussed further in section~\ref{sect:DiffNet} below

\section{Continuum Networks}\label{sec:contNet}
Motivated by the previous section, we aim to build network architectures based on diffusion processes. 
We first discuss the notion of (neural) networks in a continuum 
setting for which we introduce the concept of a \emph{continuum 
network} as a mapping between function spaces. That is, given a 
function on a bounded domain $\Omega\subset\R^d$ with $f\in L^p(\Omega)$, we are interested in finding a non-linear 
parametrised operator $\mathcal{H}_\Theta:L^p(\Omega)\to L^p(\Omega)$ acting on the function $f$. We will consider in the 
following the case $p\in\{1,2\}$; extensions to other spaces 
depend on the involved operations and will be the subject 
of future studies.

We will proceed by defining the essential building blocks of a 
continuum network and thence to discuss specific choices to 
obtain a continuum version of the most common convolutional 
neural networks. Based on this we will then introduce our 
proposed architecture as a diffusion network in the next chapter.

\subsection{Formulating a continuum network}
The essential building blocks of a deep neural network are 
obviously the several layers of neurons, but since these have a 
specific notion in classical neural networks, see for instance 
\cite{shalev2014}, we will not use the term of neurons to avoid 
confusion. We rather introduce the concept of layers and channels as the building blocks of a \emph{continuum network}. In this 
construction, each layer consists of a set of functions on a 
product space and each function represents a channel.
\begin{definition}[Layer and channels]
For $k\in \N_0$, let $F_k=\{f^k_1,f^k_2,\cdots, f^k_I\}$ be a set of functions $f^k_i\in L^p(\Omega)$ for $i\in\I=\{1,\dots,I\}$, 
$I\geq 1$. Then we call:  $F_k$ the layer $k$ with $I$ channels 
and corresponding index set $\I$. 
\end{definition}

The continuum network is then built by defining a relation or 
operation between layers. In the most general sense we define the concept of a layer operator for this task. 
\begin{definition}[Layer operator]
Given two layers $F_k$ and $F_{t}$, $k\neq t$, with channel index set $\I,\ \J$, respectively, we call the mapping $\mathcal{H}:\bigotimes_\I L^p(\Omega) \to \bigotimes_\J L^p(\Omega)$ with 
\[
\mathcal{H}F_{k} = F_t
\]
a layer operator. If the layer operator depends on a set of 
parameters $\Theta$, then we write $\mathcal{H}_{\Theta}$.
\end{definition}
We note that for simplicity, we will not index the set of 
parameters, i.e $\Theta$ generally stands for both all involved 
parameters of each layer separately, or the whole network.
The classical structure for layer operators follows the principle of affine linear transformations followed by a nonlinear 
operation. Ideally the affine linear transformation should be 
parameterisable by a few parameters, whereas the nonlinear 
operation is often fixed and acts pointwise. A popular choice is 
the maximum operator also called the ``Rectified Linear Unit'':
\[
\mathrm{ReLU}:L^p(\Omega)\to L^p(\Omega), \hspace{0.25 cm} f\mapsto \max(f,0).
\]

The continuum network is then given by the composition of all 
involved layer functions. For example in monochromatic imaging 
applications we typically have an input image $f_0$ and a desired output $f_K$ with several layer functions inbetween, that perform a specific task such as denoising or sharpening. In this case the input and output consists of one channel, i.e. $|F_0|=|F_K|=1$; 
consequently for colour images (in RGB) we have $|F_0|=|F_K|=3$.  

\subsection{Continuum convolutional networks}\label{sec:contCNN}
Let us now proceed to discuss a specific choice for the layer 
operator, namely convolutions. Such that we will obtain a 
continuum version of the widely used convolutional neural 
networks, which we will call here a \emph{continuum convolutional network}, to avoid confusion with the established convolutional 
neural networks (CNN). We note that similar ideas have been addressed as well in \cite{Adler2017}.

Let us further consider linearly ordered network architectures, 
that means each layer operator maps between consecutive layers. 
The essential layer operator for a continuum convolutional 
network is then given by the following definition.
\begin{definition}[Convolutional layer operator] 
Given two layers $F_{k-1}$ and $F_{k}$ with channel index set $\I,\ \J$, respectively. Let $b_j\in\R$ and  $\omega_{i,j}\in L^p(\Omega)$, with compact support in $\Omega$, be the layer operator's parameters for all $i\in\I,\ j\in \J$. We call $\calC^{(k)}_{\Theta,\varphi}$ the convolutional layer operator for layer $k$, if for each output channel
\begin{equation}\label{eqn:convLayer}
\calC^{(k)}_{\Theta,\varphi}F_{k-1}=\varphi\left[b_j + \sum_{i\in \mathcal{I}} \omega_{i,j} \ast f^{k-1}_{i} \right]=f^k_{j}, \ j \in \J,
\end{equation}
with a point-wise nonlinear operator $\varphi: L^p(\Omega)\to L^p(\Omega) $.
\end{definition}
If the layer operator does not include a nonlinearity, we write $\calC_{\Theta,Id}$. Now we can introduce the simplest 
convolutional network architecture by applying $K\geq 1$ 
convolutional layer operators consecutively. 

\begin{definition}[$K$-layer Continuum Convolutional Network]
Let $K\geq 1$, then we call the composition of $K$ convolutional 
layer operator, denoted by $\calC^K_\Theta$, a K-layer Continuum 
Convolutional Network, such that
\begin{equation}
\calC^K_{\Theta,\varphi} = {\calC^{(K)}_{\Theta,\varphi} 
\circ \dots \circ \calC^{(1)}_{\Theta,\varphi}}, \hspace{0.25 cm} \calC^K_\Theta F_0=F_K.
\end{equation}
\end{definition}
In the following we will also refer to a $K$-layer CNN as the 
practical implementation of a $K$-layer continuum convolutional 
network.
A popular network architecture that extends this simple idea is 
given by a resdiual network (ResNet) \cite{he2016}, that is based on the repetition of a 2-layer CNN with a residual connection, 
that consists of addition. That is, the network learns a series of additive updates to the input.
The underlying structure in ResNet is the repeated application of the following \emph{residual block}, given by
\begin{equation}
\calR_{\Theta,\varphi} = \calC^{(2)}_{\Theta,Id}\circ \calC^{(1)}_{\Theta,\varphi} + Id, \hspace{0.25 cm} \calR_{\Theta,\varphi}F_0=F_2.
\end{equation}
Note that the second convolutional layer does not include a 
nonlinearity. Furthermore, it is necessary that $|F_0|=|F_2|$, 
but typically it is often chosen such that $|F_0|=|F_1|=|F_2|$. 
The full continuum ResNet architecture can then be summarized as 
follows. Let $K\geq 1$, then the composition of $K$ residual 
blocks, denoted by $\calR^K_{\Theta,\varphi}$, defines a $K$-block continuum ResNet
\begin{equation}
\calR^K_{\Theta,\varphi} = \calC^{(K+1)}_{\Theta,\varphi} \circ \calR^{(K)}_{\Theta,\varphi} 
\circ \dots \circ \calR^{(1)}_{\Theta,\varphi}\circ \calC^{(0)}_{\Theta,\varphi}.
\end{equation}

Note that the two additional convolutional layers in the 
beginning and end are necessary to raise the cardinality of the 
input/output layer to the cardinality needed in the residual 
blocks. A complete K-block ResNet then consists of $2(K+1)$ 
layers. Note that in the original work \cite{he2016}, the 
network was primarily designed for an image classification task rather than image-to-image mapping that we consider here.


\section{DiffNet: discretisation and implementation
\label{sect:DiffNet}}

Next we want to establish a layer operator based on the diffusion processes discussed in chapter \ref{sec:Diff4imaging}. This means
that we now interpret the layers $F_k$ of the continuum network 
as time-states of the function $u: \Omega\times\R_+ \to \R$, 
where $u$ is a solution of the diffusion equation \eqref{eqn:diffusionEquation}. In the following we assume single 
channel networks, i.e. $|F_k|=1$ for all layers. Then we can 
associate each layer with the solution $u$ such that $F_k=u^{(k)}=u(x,t=t_k)$.
To build a network architecture based on the continuum setting,
we introduce the layer operator versions of \eqref{eq:Explicit1}, and \eqref{eqn:nonstatFilter}:
\begin{definition}[Diffusion and filtering layer operator] 
Given two layers $F_k$ and $F_{k-1}$, such that $F_k=u(x,t_k)$ and $F_{k-1}=u(x,t_{k-1})$, then a diffusion layer operator $\mathcal{D}_\Theta$, with parameters $\Theta=\{\gamma,\delta t\}$, is given by
\begin{equation}\label{eqn:diffLayerFunc}
\calD_\Theta F_{k-1} = \mathcal{D}_{\delta t}^{\rm Expl}(\gamma^{(k-1)})u^{(k-1)}  = (Id + \delta t \mathcal{L}(\gamma^{(k)})) u^{(k-1)} = F_k.
\end{equation}
Similarly, an inverse filtering layer operator, with parameters $\Theta=\{\zeta,\delta t \}$ is given by 
\begin{equation}\label{eqn:filtLayerFunc}
\mathcal{E}_\Theta F_{k-1} = \mathcal{E}_{\delta t}(\zeta^{(k-1)})u^{(k-1)}  =u^{(k-1)}- \delta t\int_{{\mathbb{R}}^d} \zeta(x,y) u^{(k-1)}(y) \rmd y =
F_k.
\end{equation}
\end{definition}
Note that this formulation includes a learnable time-step and 
hence the time instances that each layer represents changes. That also means that a stable step size is implicitly learned, if 
there are enough layers. In the following we discuss a few 
options on the implementation of the above layer operator, 
depending on the type of diffusivity.
\begin{remark}
The assumption of a single channel network, i.e. $|F_k|=1$ for 
all $k$, can be relaxed easily. Either by assuming $|F_k|=m$ for 
some $m \in \N$ and all layers, or by introducing a channel 
mixing as in the convolutional operator \eqref{eqn:convLayer}. 
\end{remark}

\subsection{Discretisation of a continuum network}
Let us briefly discuss some aspects on the discretisation of a 
continuum network; let us first start with affine linear 
networks, such as the convolutional networks discussed in section \ref{sec:contCNN}. Rather than discussing the computational 
implementation of a CNN, (see e.g. the comprehensive description in \cite{Goodfellow2016}), we concentrate instead on an algebraic matrix-vector formulation that serves our purposes.

For simplicity we concentrate on the two-dimensional $d=2$ case here. 
Let us then assume that the functions $f_i$ in each layer are 
represented as a square $n$-by-$n$ image and we denote the 
vectorised form as $\mathbf{f}\in \R^{n^2}$. Then any linear 
operation on $\mathbf{f}$ can be represented by some matrix $\mathsf{A}$; in particular we can represent convolutions as a 
matrix. 

We now proceed by rewriting the convolutional operation \eqref{eqn:convLayer} in matrix-vector notation. Given two layers $F_k$ and $F_{k-1}$ with channel index set $\I,\ \J$, 
respectively, then we can represent the input layer as vector $\mathbf{F}_{k-1}\in \R^{Jn^2}$ and similarly the output layer $\mathbf{F}_{k}\in \R^{In^2}$. Let $\mathsf{A}_i\in \R^{n^2\times Jn^2 }$ represent the sum of convolutions in \eqref{eqn:convLayer}, then we can write the layer operator in 
the discrete setting as matrix-vector operation  by
\[
\mathbf{f}_i^k=\varphi(b_i{\id}+\mathsf{A}_i \mathbf{F}_{k-1}),
\]
where $\id$ denotes the identity matrix of suitable size. 
Furthermore, following the above notation, we can introduce a 
stacked matrix $\mathsf{A}\in\R^{In^2\times Jn^2 }$ consisting of all $\mathsf{A}_i$ and a matrix $\mathsf{B}\in\R^{In^2\times Jn^2 }$ by stacking the diagonal matrices $b_i{\id}$. Then we can 
represent the whole convolutional layer in the discrete setting as
\begin{equation}\label{eqn:matrixVecCNN}
\mathbf{F}_k=\varphi(\mathsf{B}+\mathsf{A}\mathbf{F}_{k-1}).
\end{equation}
Now the parameters of each layer are contained in the two matrices $\mathsf{A}$ and $\mathsf{B}$.

\subsection{Learned Forward and Inverse Operators}\label{sec:learnedForwInvOp}

Let us now discuss a similar construction for the diffusion layers. 
For the implementation of the diffusion network, we 
consider the explicit formulation in \eqref{eq:Explicit1} with 
the differential operator $\mathcal{L}(\gamma^{(k)}) = \diver \gamma^{(k)}\grad$ approximated by the stencil (including the time-step $\delta t$)
\begin{align}
\mathsf{L}_{\delta t}(\gamma^{(k)})= \delta t \begin{pmatrix} & \gamma_1^{(k)}&\\ \gamma_2^{(k)}&-\sum_i\gamma_i^{(k)} &\gamma_4^{(k)} \\ &\gamma_3^{(k)}& \end{pmatrix}.
\end{align}
We use zero Neumann boundary conditions on the domain boundary
\begin{equation}
\partial_\nu u \hspace{0.5em} =\ 0  \text{ on } \partial\Omega\times (0,T].
\label{eq:NeumannBCs}
\end{equation}
Then we can represent \eqref{eqn:diffLayerFunc} by 
\begin{equation}\label{eqn:diffDiscUpdate}
\mathbf{F}_{k} = (\id + \mathsf{L}_{ \delta t}(\gamma)) \mathbf{F}_{k-1}.
\end{equation}
The basis of learning a diffusion network is now given as 
estimating the diagonals of $\mathsf{L}_{\delta t}(\gamma)$ and the time-step $\delta t$. This can be done either explicitly as for the CNN or indirectly by an 
estimator network, as we will discuss next.

\subsection{Formulating DiffNet}\label{sec:formulatingDiffNet}

Let us first note, that  if we construct our network by strictly 
following the update in \eqref{eqn:diffDiscUpdate}, we restrict 
ourselves to the forward diffusion. 
To generalise to the inverse problem we consider 
\begin{equation}
 {\mathsf{W}}_{\delta t}(\zeta)  = 
\delta t \begin{pmatrix} &\zeta_1 &\\ \zeta_2& -\zeta_5 &\zeta_4 \\ & \zeta_3 & \end{pmatrix}. 
\end{equation}

Additionally, there are two fundamentally different cases for the diffusivity $\gamma$ we need to consider before formulating a network architecture to capture the underlying behaviour. These two cases are
\begin{itemize}
    \item[i.)]  Linear diffusion; spatially varying and possible time dependence, $\gamma=\gamma(x,t)$.
    \item[ii.)] Nonlinear diffusion; diffusivity depending on the solution $u$, $\gamma=\gamma(u(x,t))$.
\end{itemize}
In the first case, we could simply to try learn the diffusivity 
explicitly, to reproduce the diffusion process. In the second 
case, this is not possible and hence an estimation of the 
diffusivity needs to be performed separately in each time step 
from the image itself, before the diffusion step can be 
performed. This leads to two conceptually different network 
architectures.

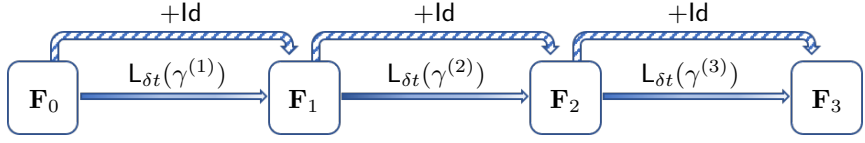
\begin{figure}[h!]
\centering
\begin{picture}(320,60)
\put(0,0){\includegraphics[width=320pt]{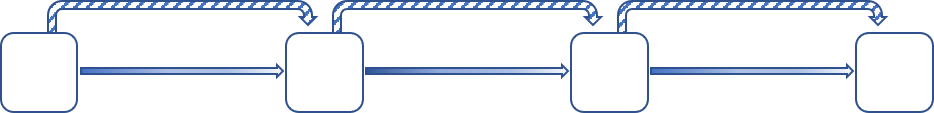}}
\put(105,60){Linear Diffusion Network }
\put(8,10){$\mathbf{F}_0$}
\put(105,10){$\mathbf{F}_1$}
\put(203,10){$\mathbf{F}_2$}
\put(301,10){$\mathbf{F}_3$}

\put(45,20){$\mathsf{L}_{\delta t}(\gamma^{(1)})$}
\put(142,20){$\mathsf{L}_{\delta t}(\gamma^{(2)})$}
\put(237,20){$\mathsf{L}_{\delta t}(\gamma^{(3)})$}
\put(57,43){$+ \id$}
\put(152,43){$+ \id$}
\put(247,43){$+ \id$}

\end{picture}
\caption{\label{fig:linDiffNet} Illustration for a linear 3-layer Diffusion network. In this case we learn the filters $\gamma$ as 
the diffusivity for each layer explicitly. }
\end{figure}

The linear case i.) corresponds to the diffusion layer operator \eqref{eqn:diffLayerFunc} and is aimed to reproduce a linear 
diffusion process with fixed diffusivity. Thus, learning the 
mean-free filter suffices to capture the physics. The resulting 
network architecture is outlined in Figure \ref{fig:linDiffNet}. 
Here, the learned filters can be directly interpreted as the 
diffusivity of layer $k$ and are then applied to $\mathbf{F}_{k-1}$ to produce $\mathbf{F}_k$.

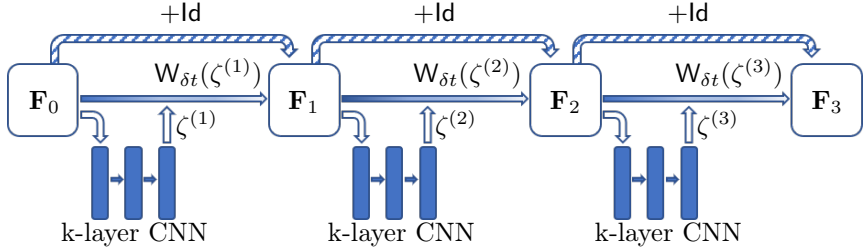
\begin{figure}[h!]
\centering
\begin{picture}(320,95)
\put(0,5){\includegraphics[width=320pt]{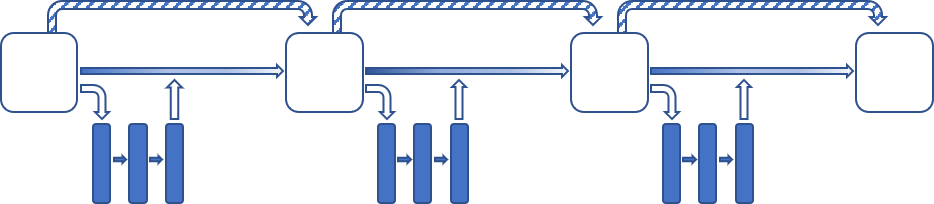}}

\put(80,95){Nonlinear Diffusion Network (DiffNet) }
\put(8,47){$\mathbf{F}_0$}
\put(105,47){$\mathbf{F}_1$}
\put(203,47){$\mathbf{F}_2$}
\put(301,47){$\mathbf{F}_3$}

\put(55,57){$\mathsf{W}_{\delta t}(\zeta^{(1)})$}
\put(63,37.5){$\zeta^{(1)}$}
\put(152,57){$\mathsf{W}_{\delta t}(\zeta^{(2)})$}
\put(160,37.5){$\zeta^{(2)}$}
\put(250,57){$\mathsf{W}_{\delta t}(\zeta^{(3)})$}
\put(258,37.5){$\zeta^{(3)}$}
\put(57,80){$+ \id$}
\put(152,80){$+ \id$}
\put(247,80){$+ \id$}

\put(20,-3){k-layer CNN}
\put(120,-3){k-layer CNN}
\put(220,-3){k-layer CNN}



\end{picture}
\caption{\label{fig:nonLinDiffNet} Illustration for a nonlinear 3-layer Diffusion network. Here the filters $\zeta$ are implicitly estimated by a small $k$-layer CNN and then applied to the image in the filtering layer.}
\end{figure}

In the nonlinear case ii.) we follow the same update structure, 
but now the filters are not learned explicitly, they are rather 
estimated from the input layer itself, as illustrated in Figure \ref{fig:nonLinDiffNet}. Furthermore, since this architecture is 
designed for inversion of the nonlinear diffusion process, we 
employ the generalised stencil  ${\mathsf{W}}_{\delta t}(\zeta)$. Then, given layer $\mathbf{F}_k$, the filters $\zeta$ are 
estimated by a small CNN from $\mathbf{F}_{k-1}$, which are then 
applied following an explicit update as in \eqref{eqn:diffDiscUpdate} to produce $\mathbf{F}_k$. Note that 
the diagonals in the update matrix are produced by the estimation CNN. We will refer to this nonlinear filtering architecture as 
the \emph{DiffNet} under consideration in the rest of this study.

\subsubsection{Implementation}\label{sec:implementation}
The essential part for the implementation of a diffusion network 
is to perform the update \eqref{eqn:diffDiscUpdate} with either $\mathsf{L}_{\delta t}(\gamma)$ or $\mathsf{W}_{\delta t}(\zeta)$. For computational reasons it is not practical to 
build the sparse diagonal matrix and evaluate \eqref{eqn:diffDiscUpdate}, we rather represent the filters $\gamma$ and $\zeta$ as an $n\times n$-image and apply the 
filters as pointwise matrix-matrix multiplication to a shifted 
and cropped image, according to the position in the stencil. This way, the zero Neumann boundary condition 
\eqref{eq:NeumannBCs}
is also automatically incorporated. 

For the linear diffusion network, we would need to learn the 
parameter set $\Theta$, consisting of filters and time steps, 
explicitly. This has the advantage of learning a global 
operation on the image where all parameters are interpretable, 
but it comes with a few disadvantages. First of all, in this 
form we are limited to linear diffusion processes and a fixed 
image size. Furthermore, the parameters grow with the image 
size, i.e. for an image of size $n\times n$ we need $5 n^2$ 
parameters per layer. Thus, applications may be limited.

For the nonlinear architecture of DiffNet, where the filters 
depend on the image at each time step, we introduced an 
additional estimator consisting of a $K$-layer CNN. This CNN 
gets an image, given as layer $\mathbf{F}_k$, as input and 
estimates the filters $\zeta$. The architecture for this $K$-layer CNN as estimator is chosen to be rather simplistic, as 
illustrated in Figure \ref{fig:DiffEstimator}. The input $\mathbf{F}_k$ consists of one channel, which is processed by $K-1$ convolutional layers with 32 channels and a ReLU 
nonlinearity, followed by the last layer without nonlinearity 
and 5 channels for each filter, which are represented as 
matrices of the same size as the input $\mathbf{F}_k$. In 
particular for the chosen filter size of $3\times 3$ we have 
exactly $9\cdot(32+32\cdot 5 + 32^2\cdot(K-2))$ convolutional 
parameters and $32\cdot(K-1)+5$ biases per diffusion layer. That 
is for a 5-layer CNN 29.509 parameters independent of image size.

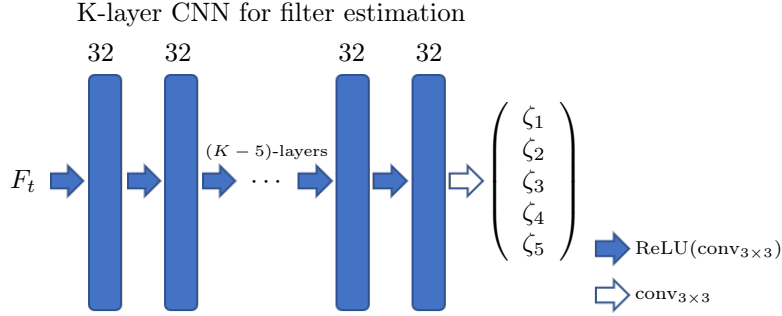
\begin{figure}[h!]
\centering
\begin{picture}(250,110)

\put(15,-5){\includegraphics[height=90pt]{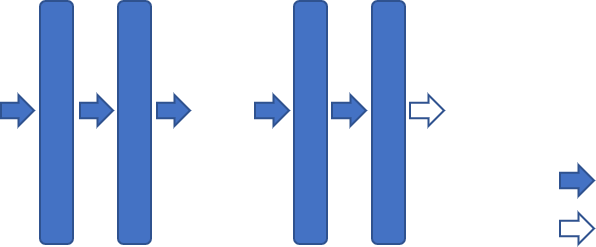}}

\put(25,105){K-layer CNN for filter estimation}
\put(0,42){$F_t$}

\put(90,42){$\cdots$}
\put(73,54){\tiny{$(K-5)$-layers}}

\put(29,90){32}
\put(58,90){32}
\put(123,90){32}
\put(152,90){32}
\put(178,42){$\left(\begin{array}{c}
     \zeta_1 \\
     \zeta_2  \\
     \zeta_3  \\
     \zeta_4  \\
     \zeta_5  
\end{array}\right)$}
\put(234,0){\scriptsize{conv$_{3\times 3}$} }
\put(234,17){\scriptsize{ReLU(conv$_{3\times 3}$)} }

\end{picture}
\caption{\label{fig:DiffEstimator} Architecture of the $K$-layer CNN used as diffusivity estimator in the nonlinear diffusion network (DiffNet). }
\end{figure}

\section{Computational experiments}\label{sec:compExp}
In the following we will examine the reconstruction capabilities 
of the proposed DiffNet. The experiments are divided into a 
simple case of deconvolution, where we can examine the learned 
features, and a more challenging problem of recovering an image 
from its nonlinear diffused and noise corrupted version.
\subsection{Deconvolution with DiffNet}
We first examine a simple deconvolution experiment to determine 
what features the DiffNet learns in an inverse problem. For this 
task we will only consider deconvolution without noise.

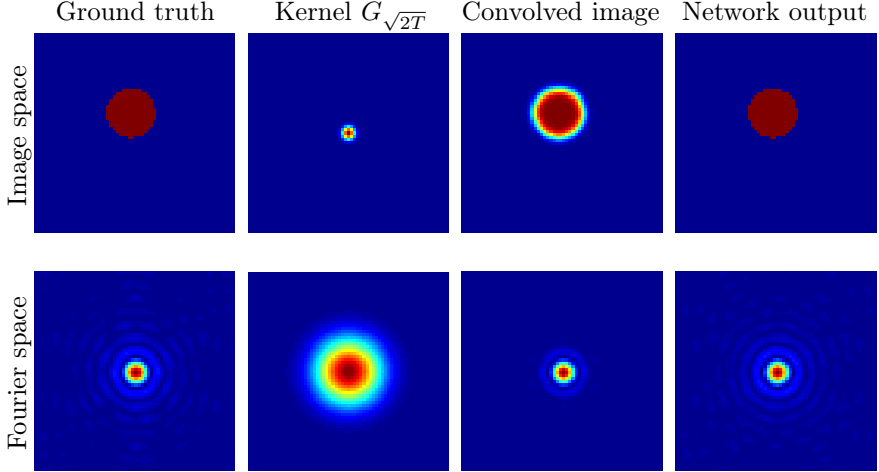
\begin{figure}[h!]
\centering
\begin{picture}(320,175)
\put(0,90){\includegraphics[width=75pt]{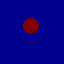}}
\put(80,90){\includegraphics[width=75pt]{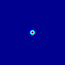}}
\put(160,90){\includegraphics[width=75pt]{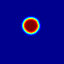}}
\put(240,90){\includegraphics[width=75pt]{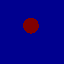}}

\put(0,0){\includegraphics[width=75pt]{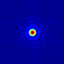}}
\put(80,0){\includegraphics[width=75pt]{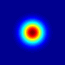}}
\put(160,0){\includegraphics[width=75pt]{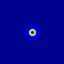}}
\put(240,0){\includegraphics[width=75pt]{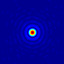}}


\put(8,170){Ground truth}
\put(90,170){Kernel $G_{\sqrt{2T}}$}
\put(160,170){Convolved image}
\put(242.5,170){Network output}

\put(-10,100){\rotatebox{90}{Image space}}
\put(-10,8)  {\rotatebox{90}{Fourier space}}

\end{picture}
\caption{\label{fig:deconvExample} Illustration of the deconvolution problem for a simple ball. Top row shows the image space and the bottom row shows the corresponding absolute value of the Fourier coefficients. All images are plotted on their own scale. }
\end{figure}

The forward problem is given by \eqref{eqn:diffusionEquation} 
with zero Neumann boundary condition \eqref{eq:NeumannBCs} and 
constant diffusivity $\gamma\equiv 1$. For the experiment we 
choose $T=1$, which results in a 
small uniform blurring, 
as shown in Figure \ref{fig:deconvExample}. We 
remind that for the isotropic diffusion, the forward model is 
equivalent to convolution in space with the kernel $G_{\sqrt{2T}}$, see \eqref{eq:IsoGreensConvSoln}. As it is also 
illustrated in Figure \ref{fig:deconvExample}, convolution in 
the spatial domain is equivalent to multiplication in Fourier 
domain. In particular, high frequencies get damped and the 
convolved image is dominated by low frequencies. Hence, for the 
reconstruction task without noise, we essentially need to recover the high frequencies.

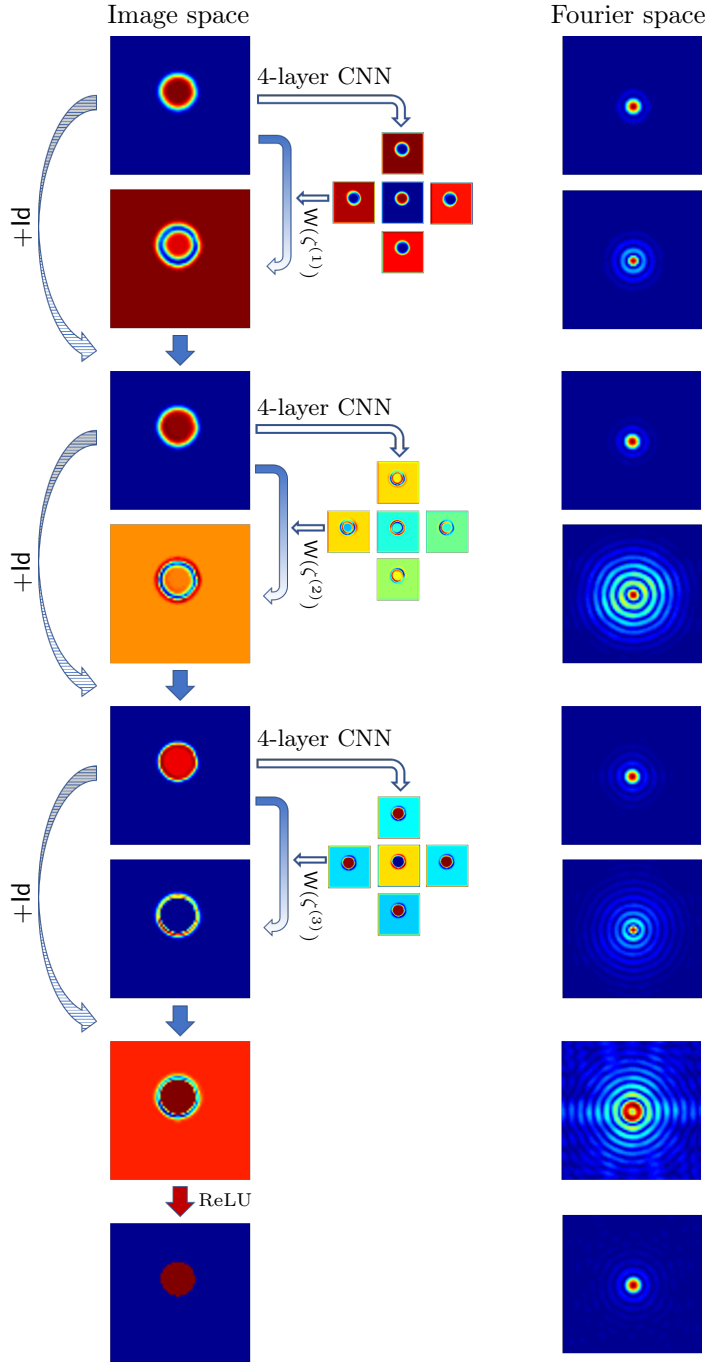
\begin{figure}[h!]
\centering
\begin{picture}(250,510)

\put(0,0){\includegraphics[height=500pt]{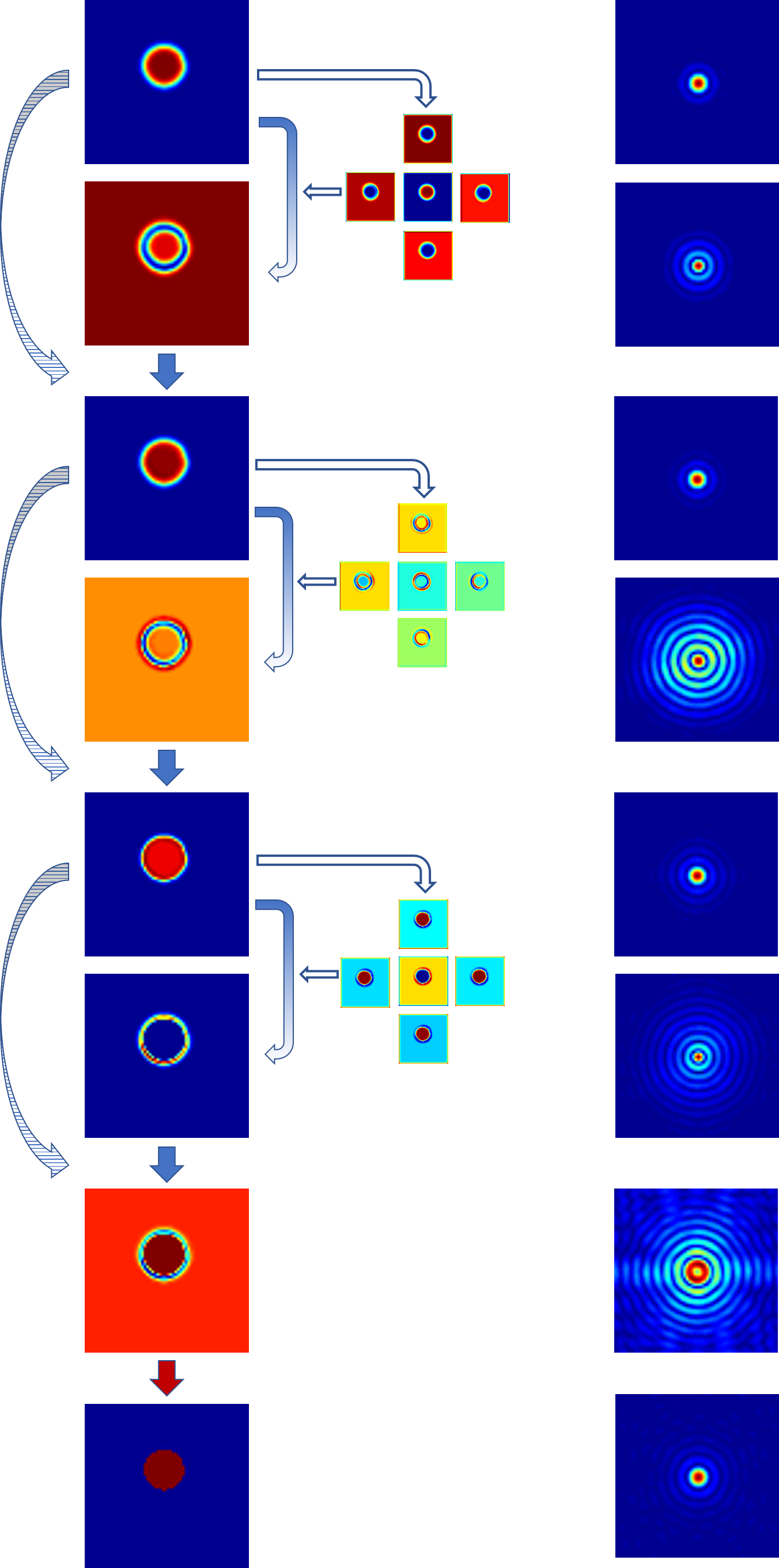}}
\put(82,482){\small{$4$-layer CNN}}
\put(82,357){\small{$4$-layer CNN}}
\put(82,232){\small{$4$-layer CNN}}

\put(-10,420){\rotatebox{90}{$+\id$}}
\put(-10,290){\rotatebox{90}{$+\id$}}
\put(-10,165){\rotatebox{90}{$+\id$}}

\put(60,58.5){\scriptsize{ReLU}}

\put(98,435){\rotatebox{-90}{\scriptsize{$\mathsf{W}(\zeta^{(1)})$}}}
\put(98,310){\rotatebox{-90}{\scriptsize{$\mathsf{W}(\zeta^{(2)})$}}}
\put(98,185){\rotatebox{-90}{\scriptsize{$\mathsf{W}(\zeta^{(3)})$}}}

\put(26,505){Image space}
\put(192,505){Fourier space}

\end{picture}
\caption{\label{fig:deconvDiffnet} Illustration of the deconvolution process with three layers of DiffNet. The left column shows the processed image and intermediate steps. The right column shows the corresponding absolute value of Fourier coefficients. All images are plotted on their own scale. }
\end{figure}

The training and test data for DiffNet consists of simple disks 
of varying radius and contrast. The training set consists of 
1024 samples and the test set of an additional 128, each of size $64\times 64$. The network architecture is chosen following the 
schematic in Figure \ref{fig:nonLinDiffNet}, with 3 diffusion 
layers and a final projection to the positive numbers by a ReLU 
layer. The filter estimator is given by a 4-layer CNN, as 
described in section \ref{sec:implementation}. All networks were 
implemented in Python with TensorFlow \cite{tensorflow2015}. 

The input to the network is given by the convolved image without 
noise and we have minimised the $\ell^2$-loss of the output to 
the ground-truth image. The optimisation is performed for about 
1000 epochs in batches of 16 with the Adam algorithm and initial 
learning rate of $4\cdot10^{-4}$ and a gradual decrease to $10^{-6}$. Training on a single Nvidia Titan Xp GPU takes about 
24 minutes. The final training and test error is both at a PSNR 
of 86.24, which corresponds to a relative $\ell^2$-error of $2.5\cdot 10^{-4}$. We remind that this experiment was performed 
without noise.

The result of the network and intermediate updates for one 
example from the test data are illustrated in Figure \ref{fig:deconvDiffnet}. 
We also show the filters $\zeta^{(k)}$ computed as the output
of the trained CNN in each layer, $k = 1\ldots3$.
The output of the last diffusion layer $\mathbf{F}_3$ is 
additionally processed by a ReLU layer to enforce positivity in 
the final result. It can be seen that the network gradually 
reintroduces the high frequencies in the Fourier domain; 
especially the last layer mainly reintroduces the high 
frequencies to the reconstruction. It is interesting to see, 
that the learned filters follow indeed the convention that the 
central filter is of different sign than the directional 
filters. This enforces the assumption that the filter consists 
of a mean free part and a regularising part, which should be 
small in this case, since we do not have any noise in the data. 
Lastly, we note that the final layer, before projection to the 
positive numbers, has a clearly negative part around the target, 
which will be cut off resulting in a sharp reconstruction of the 
ball.

\subsection{Nonlinear diffusion}\label{sec:expNonlinDiff}
Let us now consider the nonlinear diffusion process with the Perona-Malik filter function \cite{Perona1990} for \eqref{eqn:diffusionEquation} with zero Neumann boundary condition \eqref{eq:NeumannBCs}. In this model the diffusivity is given as a function of the gradient
\begin{equation}
    \gamma(|\grad u|^2) = \frac{1}{1+|\grad u|^2/\lambda^2}
\end{equation}
with contrast parameter $\lambda > 0$. We mainly concentrate 
here on the inverse problem of restoring an image that has been 
diffused with the Perona-Malik filter and contaminated by noise.

For the experiments we have used the test data from the STL-10 
database \cite{coates2011}, which consists of 100,000 RGB images 
with resolution $96\times96$. These images have been converted 
to grayscale and divided to 90,000 for training and 10,000 for 
testing. The obtained images were then diffused for 4 time steps 
with $\delta t = 0.1$ and $\lambda=0.2$. A few sample images 
from the test data with the result of the diffusion are displayed in Figure \ref{fig:STL10_forwInv}. The task is then to revert 
the diffusion process with additional regularisation to deal 
with noise in the data.

\begin{figure}[h!]
\centering
\begin{picture}(380,200)
\put(0,130){\includegraphics[width=60pt]{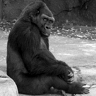}}
\put(65,130){\includegraphics[width=60pt]{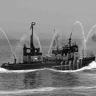}}
\put(130,130){\includegraphics[width=60pt]{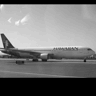}}
\put(195,130){\includegraphics[width=60pt]{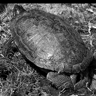}}
\put(260,130){\includegraphics[width=60pt]{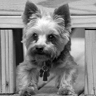}}
\put(325,130){\includegraphics[width=60pt]{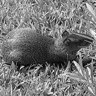}}

\put(0,65){\includegraphics[width=60pt]{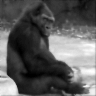}}
\put(65,65){\includegraphics[width=60pt]{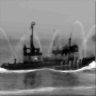}}
\put(130,65){\includegraphics[width=60pt]{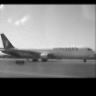}}
\put(195,65){\includegraphics[width=60pt]{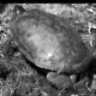}}
\put(260,65){\includegraphics[width=60pt]{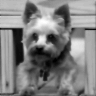}}
\put(325,65){\includegraphics[width=60pt]{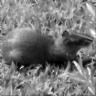}}

\put(0,0){\includegraphics[width=60pt]{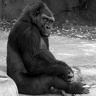}}
\put(65,0){\includegraphics[width=60pt]{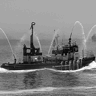}}
\put(130,0){\includegraphics[width=60pt]{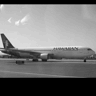}}
\put(195,0){\includegraphics[width=60pt]{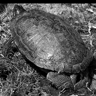}}
\put(260,0){\includegraphics[width=60pt]{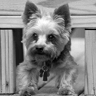}}
\put(325,0){\includegraphics[width=60pt]{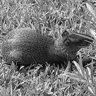}}

\put(-10,140){\rotatebox{90}{Original}}
\put(-10,75){\rotatebox{90}{Diffused}}
\put(-18,12){\rotatebox{90}{Inverted}}
\put(-9,15){\rotatebox{90}{DiffNet}}

\end{picture}
\caption{\label{fig:STL10_forwInv} Samples from the test data for learning the inversion of nonlinear diffusion without noise. Mean PSNR for reconstructed test data with DiffNet is:  63.72}
\end{figure}

\begin{figure}[h!]
\centering
\begin{picture}(380,130)

\put(0,65){\includegraphics[width=60pt]{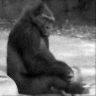}}
\put(65,65){\includegraphics[width=60pt]{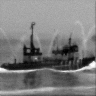}}
\put(130,65){\includegraphics[width=60pt]{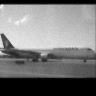}}
\put(195,65){\includegraphics[width=60pt]{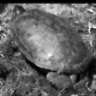}}
\put(260,65){\includegraphics[width=60pt]{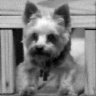}}
\put(325,65){\includegraphics[width=60pt]{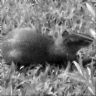}}

\put(0,0){\includegraphics[width=60pt]{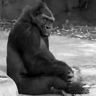}}
\put(65,0){\includegraphics[width=60pt]{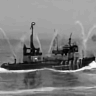}}
\put(130,0){\includegraphics[width=60pt]{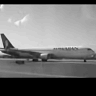}}
\put(195,0){\includegraphics[width=60pt]{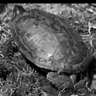}}
\put(260,0){\includegraphics[width=60pt]{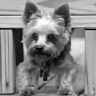}}
\put(325,0){\includegraphics[width=60pt]{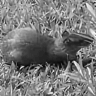}}

\put(-18,75){\rotatebox{90}{Diffused}}
\put(-9,72){\rotatebox{90}{$1\%$ Noise}}
\put(-18,12){\rotatebox{90}{Inverted}}
\put(-9,15){\rotatebox{90}{DiffNet}}

\end{picture}
\caption{\label{fig:STL10_InvNoise} Samples from the test data 
for learning the inversion of nonlinear diffusion with $1\%$ 
noise. Mean PSNR for reconstructed test data with DiffNet is:     34.21}
\end{figure}

For all experiments we have used the same network architecture 
of DiffNet using the architecture illustrated in Figure \ref{fig:nonLinDiffNet}. By performing initial tests on the 
inversion without noise, we have found that 5 diffusion layers 
with a 4-layer CNN, following the architecture in \ref{fig:DiffEstimator}, gave the best trade-off between 
reconstruction quality and network size. Increasing the amount 
of either layers, led to minimal increase in performance. 
Additionally, we have used a ReLU layer at the end to enforce 
non-negativity of the output, similarly to the last experiment. 
We emphasise that this architecture was used for all experiments 
and hence some improvements for the high noise cases might be 
expected with more layers. All networks were trained for 18 
epochs, with a batch size of 16, and $\ell^2$-loss. For the 
optimisation we have used the Adam algorithm with initial 
learning rate of $2\cdot 10^{-3}$ and a gradual decrease to $4\cdot 10^{-6}$. Training on a single Nvidia Titan Xp GPU takes about 75 minutes.

As benchmark, we have performed the same experiments with a 
widely used network architecture known as U-Net \cite{Ronneberger2015}. This architecture has been widely 
applied in inverse problems \cite{antholzer2018,hamilton2018TMI,Jin2017,Kang2017} and is
mainly used to post-process initial directly reconstructed 
images from undersampled or noisy data. For instance by filtered 
back-projection in X-ray tomography or the inverse Fourier 
transform in magnetic resonance imaging \cite{hauptmann2018MRM}. 
The network architecture we are using follows the publication \cite{Jin2017} and differs from the original mainly by a 
residual connection at the end. That means the network is 
trained to remove noise and undersampling artefacts from the 
initial reconstruction. In our context, the network needs to 
learn how to remove noise and reintroduce edges. For training 
we have followed a similar protocol as for DiffNet. The 
only difference is that we started with an initial learning rate 
of $5\cdot 10^{-4}$ with a gradual decrease to $2\cdot 10^{-5}$. 
Training of U-Net takes about 3 hours.

\begin{figure}[h!]
\centering
\begin{picture}(325,225)

\put(0,115){\includegraphics[width=100pt]{PMimages/image_clean_8.png}}
\put(115,115){\includegraphics[width=100pt]{PMimages/image_inv_8.png}}
\put(225,115){\includegraphics[width=100pt]{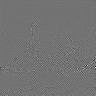}}

\put(0,0){\includegraphics[width=100pt]{PMimages/image_diff_8.png}}
\put(115,0){\includegraphics[width=100pt]{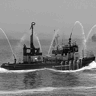}}
\put(225,0){\includegraphics[width=100pt]{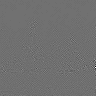}}

\put(30,220){Original}
\put(10,105){{Diffused (No Noise)}}
\put(130,220){Reconstructed}
\put(255,220){Residual}
\put(105,150){\rotatebox{90}{DiffNet}}
\put(105,35){\rotatebox{90}{U-Net}}

\end{picture}
\caption{\label{fig:reconComp} Comparison of reconstruction quality for reconstruction from  nonlinear diffused image without noise. Both networks are trained on the full set of 90,000 images. PNSR: DiffNet 65.34, U-Net 61.08}
\end{figure}

\begin{figure}[h!]
\centering
\begin{picture}(325,220)

\put(0,115){\includegraphics[width=100pt]{PMimages/image_clean_8.png}}
\put(115,115){\includegraphics[width=100pt]{PMimages/image_inv1e2_8.png}}
\put(225,115){\includegraphics[width=100pt]{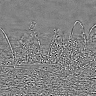}}

\put(0,0){\includegraphics[width=100pt]{PMimages/image_diff1e2_8.png}}
\put(115,0){\includegraphics[width=100pt]{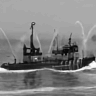}}
\put(225,0){\includegraphics[width=100pt]{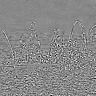}}

\put(30,220){Original}
\put(8,105){{Diffused ($1\%$ Noise)}}
\put(130,220){Reconstructed}
\put(255,220){Residual}
\put(105,150){\rotatebox{90}{DiffNet}}
\put(105,35){\rotatebox{90}{U-Net}}

\end{picture}
\caption{\label{fig:reconCompNoise} Comparison of reconstruction quality for reconstruction from 1\% noise contaminated nonlinear diffused image. Both networks are trained on the full set of 90,000 images. PNSR: DiffNet 34.96, U-Net 35.27}
\end{figure}

The reconstruction results, for some samples of the test data, 
with DiffNet can be seen in Figure \ref{fig:STL10_forwInv} for 
the case without noise and in Figure \ref{fig:STL10_InvNoise} 
for 1\% noise on the diffused images. A comparison of the 
reconstruction results with U-Net and DiffNet are shown in Figure \ref{fig:reconComp} for the test without noise and in Figure \ref{fig:reconComp} for 1\% noise. Qualitatively, the 
reconstructed images are very similar, as can be seen in the 
residual images in the last column. The leftover noise pattern 
for both networks is concentrated on the fine structures of the 
ship.
Quantitatively, for the noise free experiment DiffNet has an 
increase of 4dB in PSNR compared to the result of U-Net, 65.34 
(DiffNet) compared to 61.08 (U-Net). For the case with 1\% noise, the quantitative measures are very similar. Here U-Net has a 
slightly higher PSNR with 35.27 compared to DiffNet with 34.96. A thorough study of reconstruction quality of both networks follows in the next section as well as some interpretation of the learned features in DiffNet.

\section{Discussion}\label{sec:discussion}
First of all we note that the updates in DiffNet are performed 
explicitly and that the CNN in the architecture is only used to 
produce the filters $\zeta$. This means that DiffNet needs to 
learn a problem specific processing, in contrast to a purely data driven processing in a CNN. Consequently, the amount of necessary learnable parameters is much lower. For instance the 5-layer 
DiffNet used for inversion of the nonlinear diffusion in section \ref{sec:expNonlinDiff} has 101,310 parameters, whereas the used 
U-Net with a filter size of $3\times 3$ has a total of 34,512,705 
parameters; i.e DiffNet uses only $\sim 0.3\%$ of parameters 
compared to U-Net and hence the learned features can be seen to be much 
more explicit. In the following we discuss some aspects that 
arise from this observation, such as generalisability and 
interpretability.

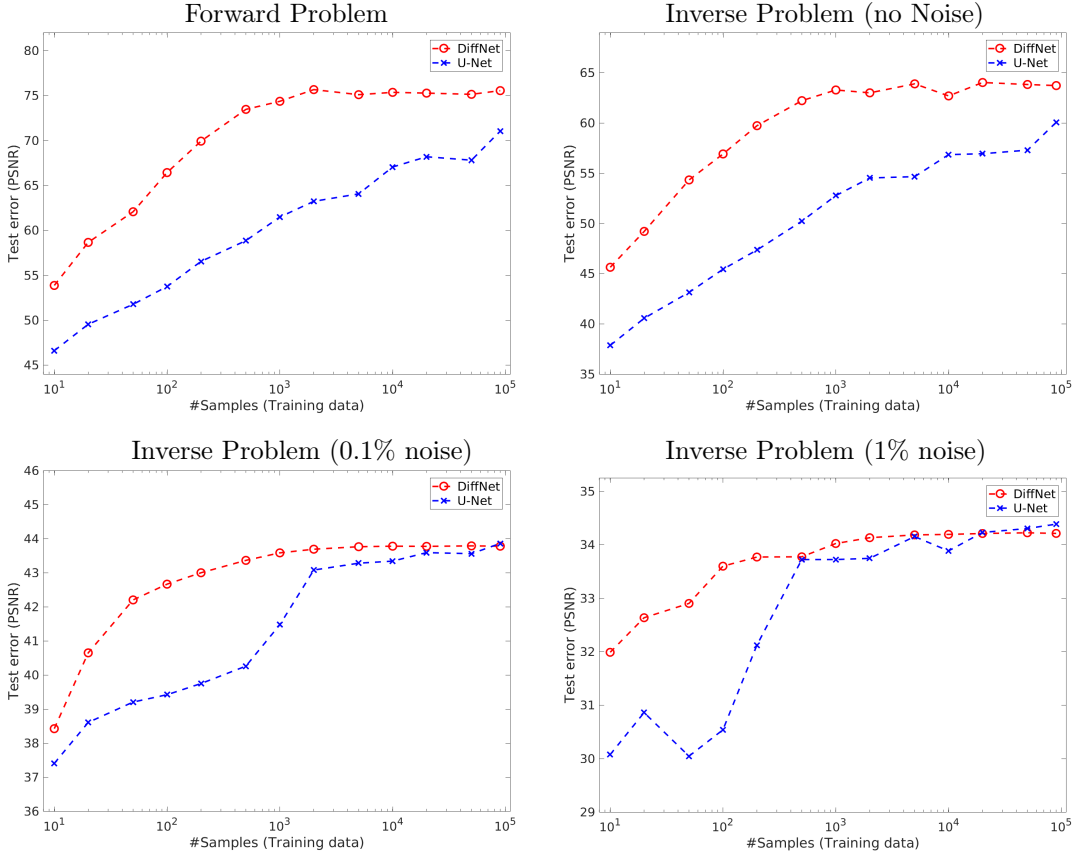
\begin{figure}[t!]
\centering
\begin{picture}(400,335)
\put(-32,160){\includegraphics[width=225pt]{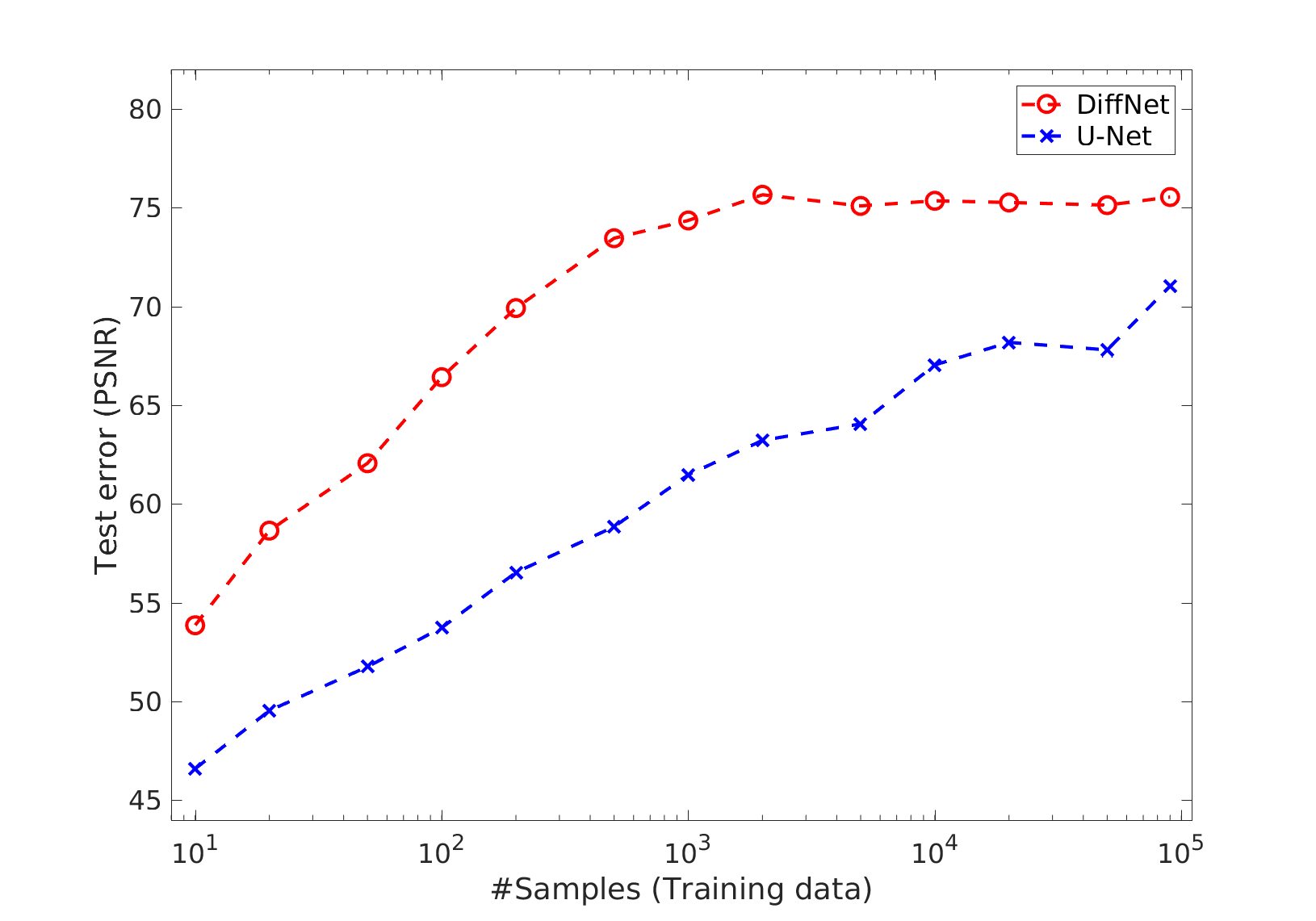}}
\put(50,310){Forward Problem}
\put(176,160){\includegraphics[width=225pt]{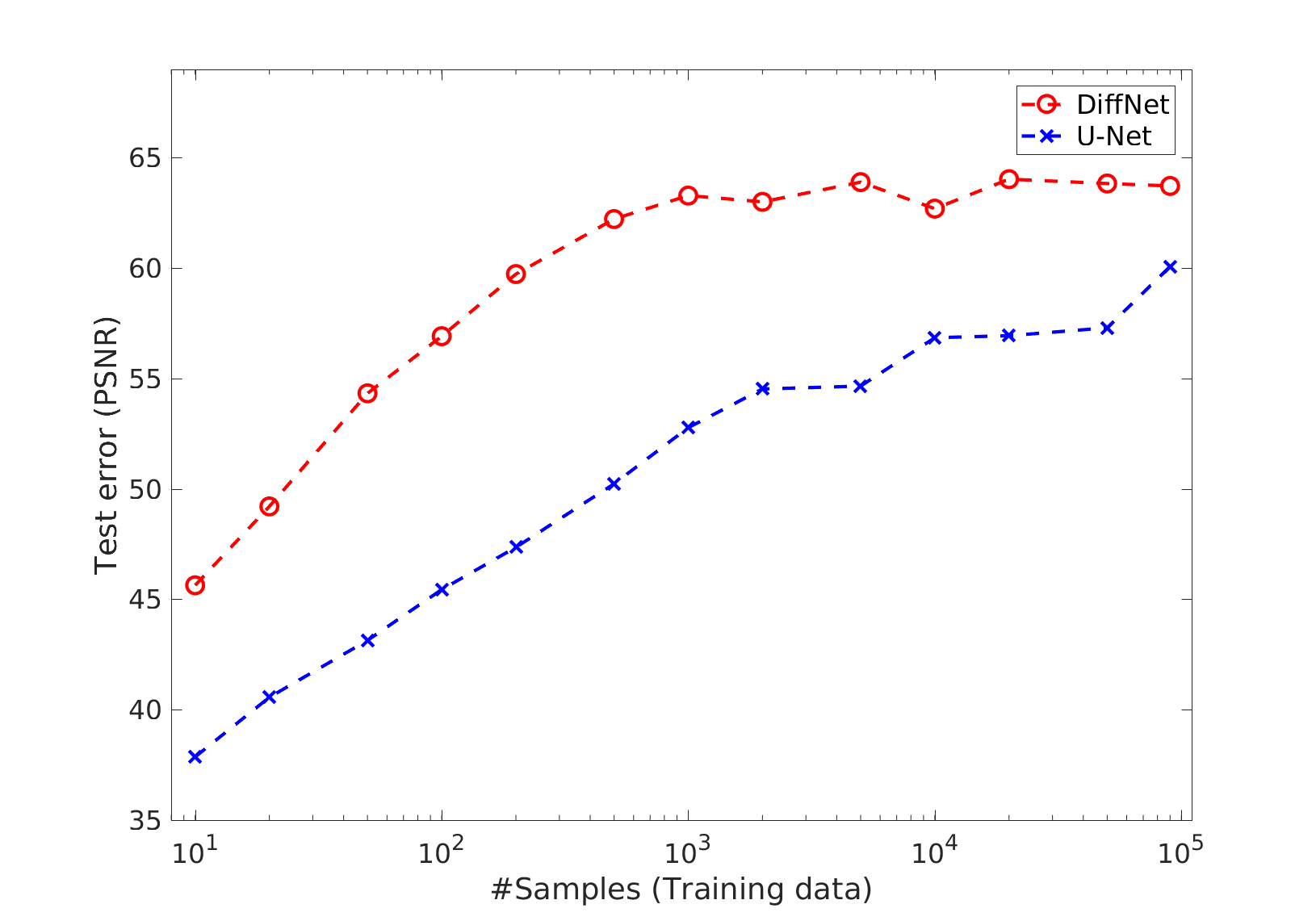}}
\put(230,310){Inverse Problem (no Noise)}

\put(-32,-5){\includegraphics[width=225pt]{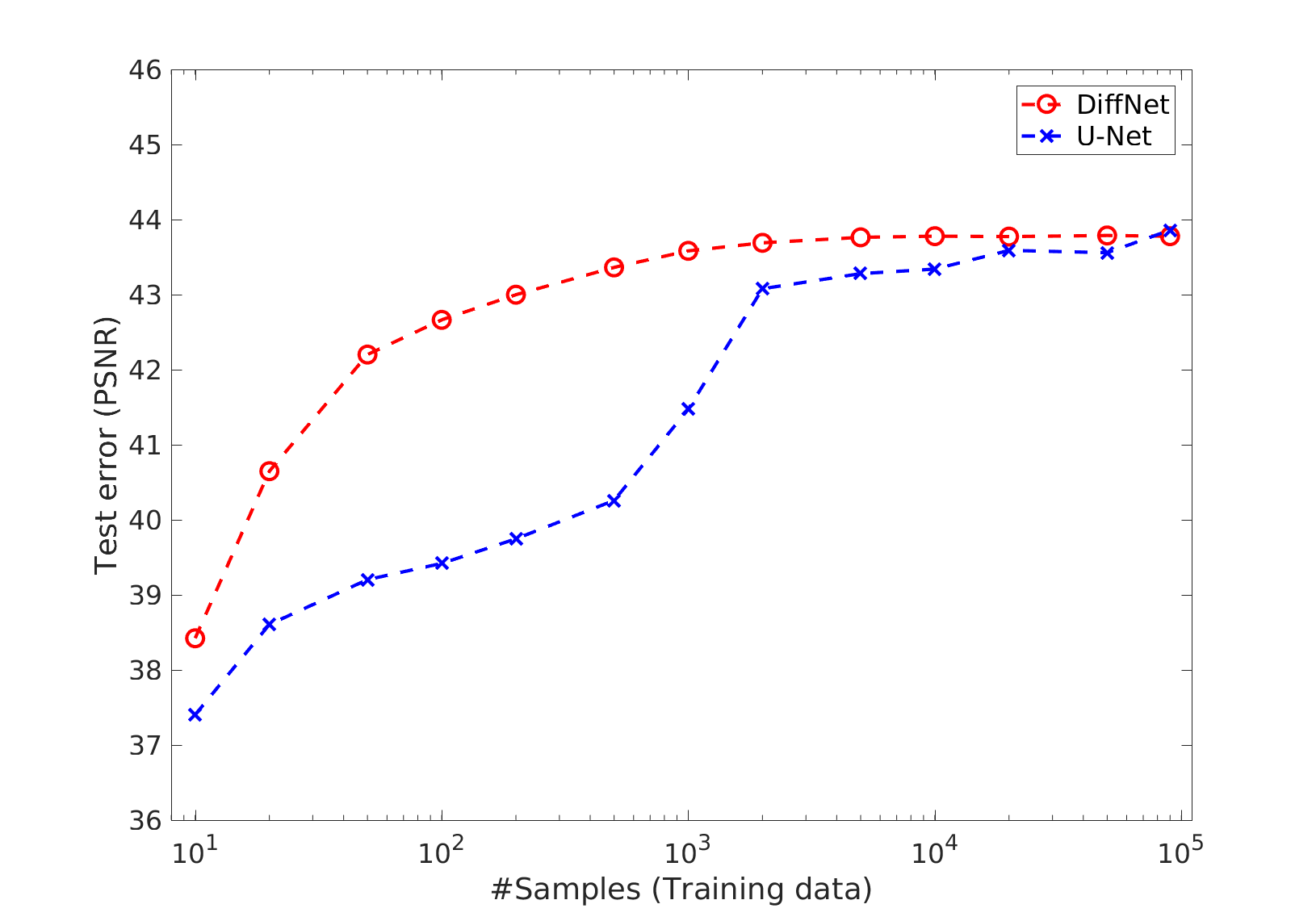}}
\put(30,145){Inverse Problem (0.1\% noise)}
\put(176,-5){\includegraphics[width=225pt]{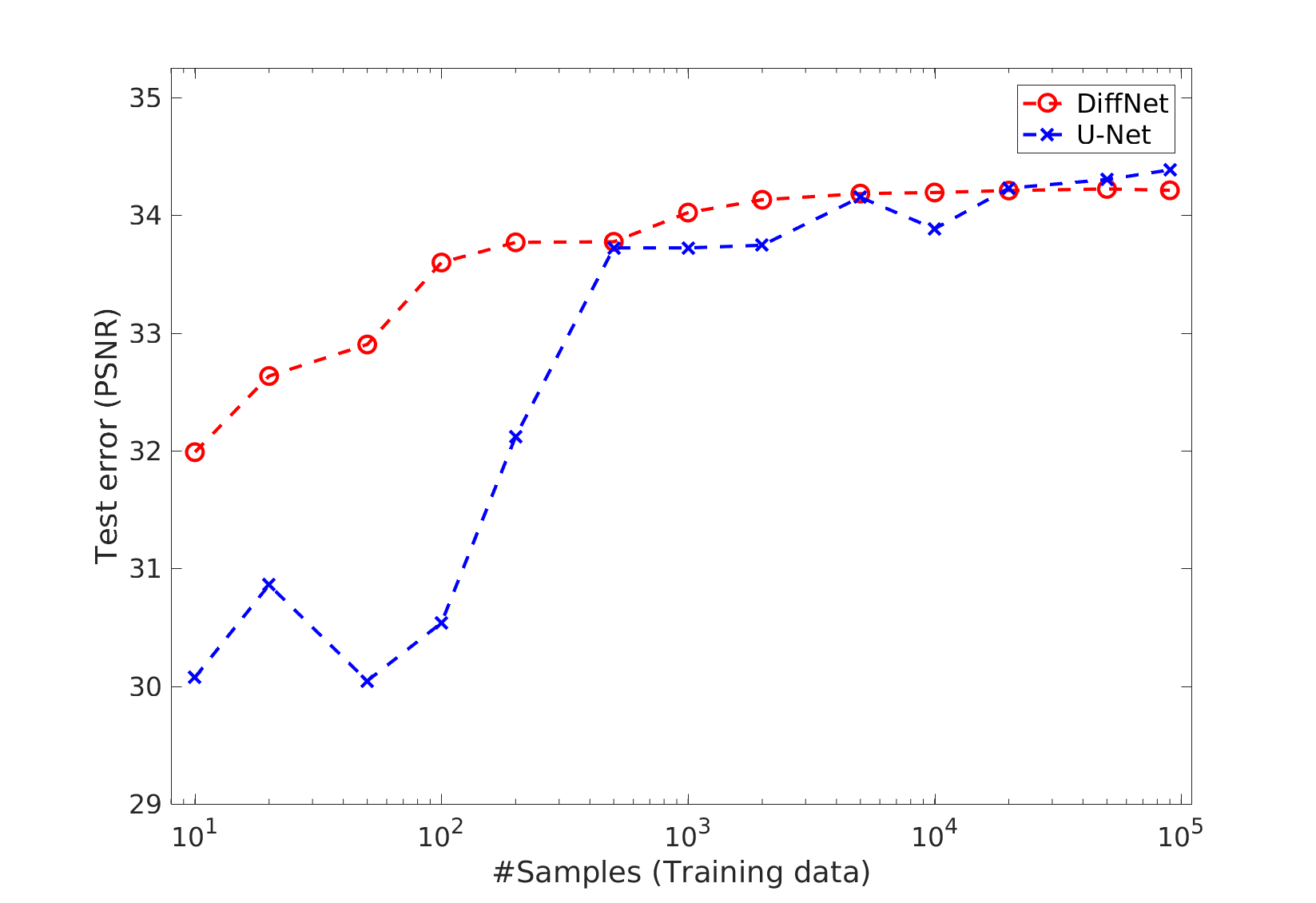}}
\put(230,145){Inverse Problem (1\% noise)}

\put(110,330){\Large{Test error vs. training size}}

\end{picture}
\caption{\label{fig:TestErrorVsSample} Generalisation plot for 
the forward and inverse problem of nonlinear diffusion and 
varying noise levels. Test error depending on the amount of 
training data, for both DiffNet and U-Net.}
\end{figure}

\subsection{Generalisability}
To test the generalisation properties of the proposed DiffNet we 
have performed similar experiments as in section \ref{sec:expNonlinDiff} for nonlinear diffusion, but with 
increasing amounts of training data. Under the assumption that 
DiffNet learns a more explicit update than a classic CNN, we 
would expect also to require less training data to achieve a good test error. 
To certify this assumption, we have examined 4 settings of 
nonlinear diffusion with the Perona-Malik filter: learning the 
forward model, learning to reconstruct from the diffused image 
without noise, as well as with 0.1\% and 1\% noise. We then 
created training data sets of increasing size from just 10 
samples up the full size of 90,000. For all scenarios we have 
trained DiffNet and U-Net following the training protocol 
described in \ref{sec:expNonlinDiff}. Additionally, we have 
aborted the procedure when the networks started to clearly 
overfit the training data.

Results for the four scenarios are shown in Figure \ref{fig:TestErrorVsSample}. Most notably DiffNet outperforms 
U-Net clearly for the forward problem and the noise free 
inversion, by 4dB and 3dB, respectively. For the noisy cases 
both networks perform very similar for the full training data 
size of 90,000. The biggest difference overall is that DiffNet 
achieves its maximum test error already with 500-1,000 samples independent of the noise case, whereas the U-Net test error saturates earlier with higher noise. 
In conclusion we can say, that for the noisy cases both networks 
are very comparable in reconstruction quality, but for small 
amounts of data the explicit nature of DiffNet is clearly 
superior.

\subsection{Interpretation of learned filters}
Since all updates are performed explicitly with the output from 
the filter estimation CNN, we can interpret some of the learned 
features. For this purpose we show the filters for the ship 
image from section \ref{sec:expNonlinDiff} for the three 
inversion scenarios under consideration. In Figure \ref{fig:meanFilters} we show the sum of all learned filters, 
i.e. $\sum_{i=1}^4 \zeta_i - \zeta_5$. If the network would only 
learn the mean-free differentiating part, then these images 
should be zero. That implies that the illustrated filters in Figure \ref{fig:meanFilters} can be related to the learned 
regularisation $\mathsf{S}(\zeta)$. Additionally, we also show 
the diagonal filters $\zeta_5$ in Figure \ref{fig:diagFilters}.

We would expect that with increasing noise level, the filters 
will incorporate more smoothing to deal with the noise; this 
implies that the edges get wider with increasing noise level. This can 
be nicely observed for the diagonal filters in Figure \ref{fig:diagFilters}. For the smoothing in Figure \ref{fig:meanFilters}, we see that the first layer consists of 
broader details and edges that are refined in the noise free 
case for increasing layers. In the noisy case the later layers 
include some smooth features that might depict the 
regularisation necessary in the inversion procedure. It is 
generally interesting to observe, that the final layer shows 
very fine local details, necessary to restore fine details for
the final output.

\begin{figure}[h!]
\centering
\begin{picture}(400,250)

\put(0,160){\includegraphics[width=75pt]{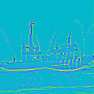}}
\put(80,160){\includegraphics[width=75pt]{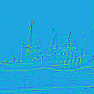}}
\put(160,160){\includegraphics[width=75pt]{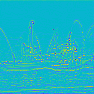}}
\put(240,160){\includegraphics[width=75pt]{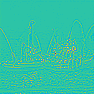}}
\put(320,160){\includegraphics[width=75pt]{PMimages/kap4S_inv_8.png}}

\put(0,80){\includegraphics[width=75pt]{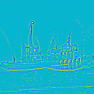}}
\put(80,80){\includegraphics[width=75pt]{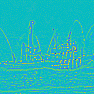}}
\put(160,80){\includegraphics[width=75pt]{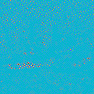}}
\put(240,80){\includegraphics[width=75pt]{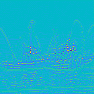}}
\put(320,80){\includegraphics[width=75pt]{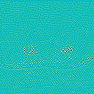}}

\put(0,0){\includegraphics[width=75pt]{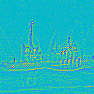}}
\put(80,0){\includegraphics[width=75pt]{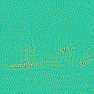}}
\put(160,0){\includegraphics[width=75pt]{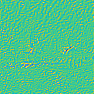}}
\put(240,0){\includegraphics[width=75pt]{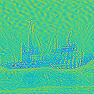}}
\put(320,0){\includegraphics[width=75pt]{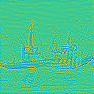}}

\put(20,240){Layer 1}
\put(100,240){Layer 2}
\put(180,240){Layer 3}
\put(260,240){Layer 4}
\put(340,240){Layer 5}
\put(-10,180){\rotatebox{90}{No Noise}}
\put(-10,90){\rotatebox{90}{$0.1\%$ Noise}}
\put(-10,20){\rotatebox{90}{$1\%$ Noise}}


\end{picture}
\caption{\label{fig:meanFilters} Filter updates $\sum_{i=1}^4 \zeta_i - \zeta_5$ for different noise levels. Each image is displayed on its own scale.}
\end{figure}

\begin{figure}[h!]
\centering
\begin{picture}(400,250)

\put(0,160){\includegraphics[width=75pt]{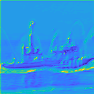}}
\put(80,160){\includegraphics[width=75pt]{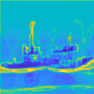}}
\put(160,160){\includegraphics[width=75pt]{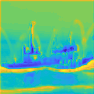}}
\put(240,160){\includegraphics[width=75pt]{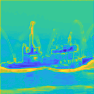}}
\put(320,160){\includegraphics[width=75pt]{PMimages/kap4C_inv_8.png}}

\put(0,80){\includegraphics[width=75pt]{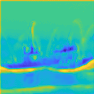}}
\put(80,80){\includegraphics[width=75pt]{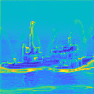}}
\put(160,80){\includegraphics[width=75pt]{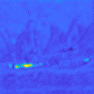}}
\put(240,80){\includegraphics[width=75pt]{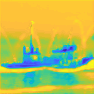}}
\put(320,80){\includegraphics[width=75pt]{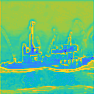}}

\put(0,0){\includegraphics[width=75pt]{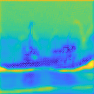}}
\put(80,0){\includegraphics[width=75pt]{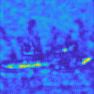}}
\put(160,0){\includegraphics[width=75pt]{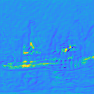}}
\put(240,0){\includegraphics[width=75pt]{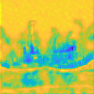}}
\put(320,0){\includegraphics[width=75pt]{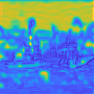}}

\put(-10,180){\rotatebox{90}{No Noise}}
\put(-10,90){\rotatebox{90}{$0.1\%$ Noise}}
\put(-10,20){\rotatebox{90}{$1\%$ Noise}}

\put(20,240){Layer 1}
\put(100,240){Layer 2}
\put(180,240){Layer 3}
\put(260,240){Layer 4}
\put(340,240){Layer 5}


\end{picture}
\caption{\label{fig:diagFilters} Obtained diagonal filters $\zeta_5$ for different noise levels. Each filter is displayed on its own scale.}
\end{figure}

Finally, we have computed training data of a wider noise range 
to examine the regularisation properties of the learned network. 
For this we have taken the full 90,000 training samples and 
contaminated the diffused image with noise ranging from $0.01\%$ 
to 10\% noise. As we conjectured in section \ref{sec:PDEdiscretisation}, the learned update filters can be 
decomposed to a mean-free operation and a smoothing part $\mathsf{W}(\zeta)=\mathsf{L}(\zeta)+\mathsf{S}(\zeta)$. 
This implies that the magnitude of $\mathsf{S}(\zeta)$ has 
to increase  with higher noise. To examine this conjecture, 
we have taken (fixed) 32 samples from the reconstructed test 
data for each noise level and computed an estimate 
of $\mathsf{S}$ as the sum $\sum_{i=1}^4 \zeta_i - \zeta_5$, i.e. the deviation from the 
mean-free part. Furthermore, we interpret the smoothing level $\alpha$ as the mean of the absolute value of $\mathsf{S}$. The 
resulting graph of smoothing versus noise level is shown in 
Figure \ref{fig:smoothingVsNoise}. As we have conjectured, the 
estimate of $\alpha$ increases clearly with the noise level and 
hence we believe our interpretation of the learned filters as 
the composition of a mean-free part and a smoothing necessary 
for ill-posed inverse problems is valid.

\begin{figure}[h!]
\centering
\begin{picture}(225,180)
\put(0,0){\includegraphics[width=250pt]{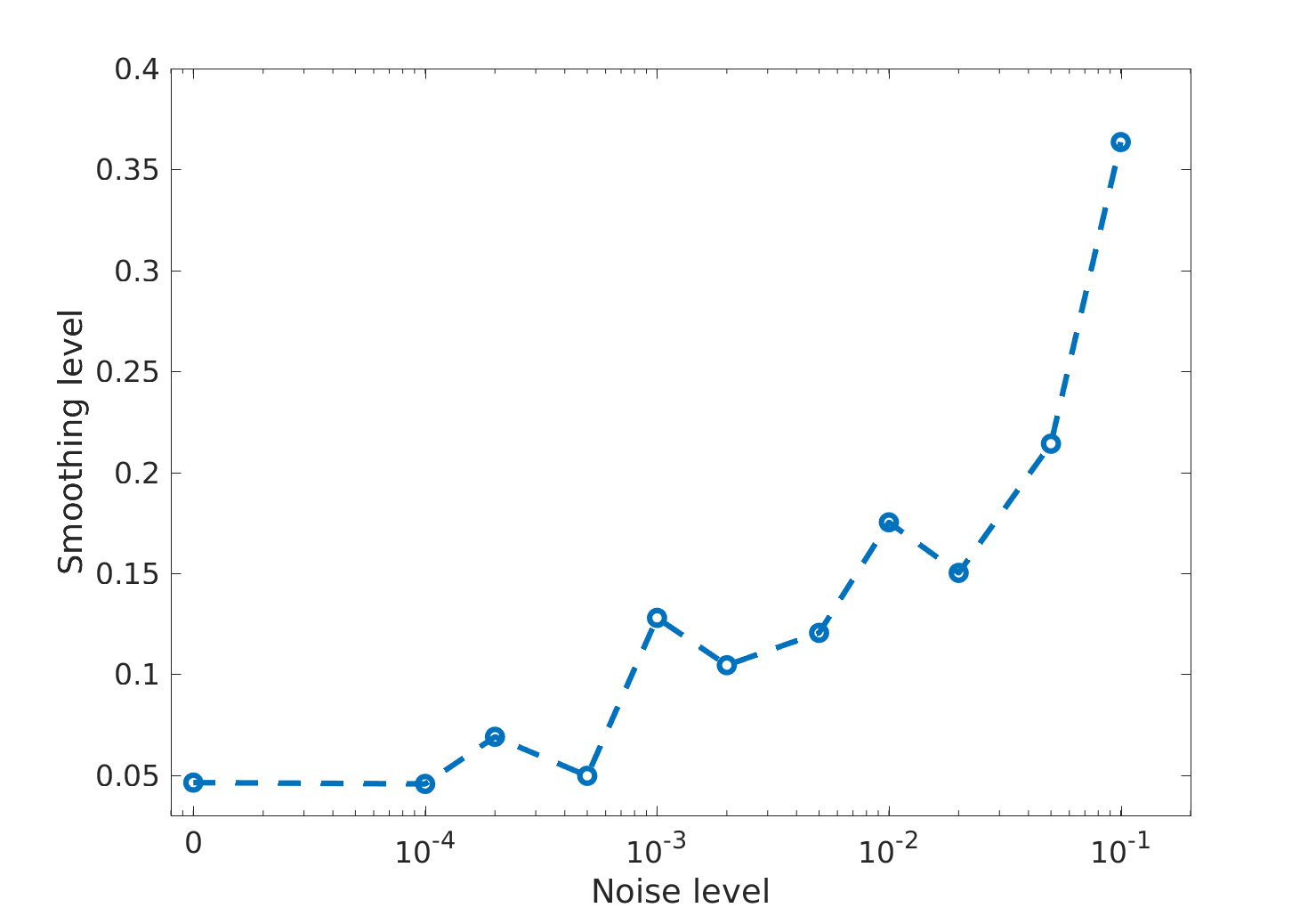}}
\put(45,165){Increasing smoothing with noise level}

\end{picture}
\caption{\label{fig:smoothingVsNoise} Estimate of the smoothing level $\alpha$ for increasing noise in the inverse problem. Computed over a sample of 32 images from the test data.}
\end{figure}






 

\section{Conclusions}\label{sec:conclusions}
In this paper we have explored the possibility to establish novel
network architectures based on physical models other than 
convolutions; in particular we concentrated here on diffusion 
processes. As main contributions, we have introduced some 
non-linear forward mappings, modelled through learning rather than just through PDEs or integral transforms. We have reviewed (regularised) 
inverse diffusion processes for inverting such maps. In 
particular, we have conjectured that these inverse diffusion 
processes can be represented by local non-stationary filters, 
which can be learned in a network architecture. More specific, 
these local filters can be represented by a sparse sub-diagonal 
(SSD) matrix and hence efficiently used in the discrete setting of a neural network. We emphasise that even though we have concentrated 
this study on a specific structure for these SSD matrices based 
on diffusion, other (higher order) models can be considered. 

We obtain higher interpretability of the network architecture, 
since the image processing is  explicitly performed by the 
application of the SSD matrices. Consequently, this means that only a 
fraction of parameters is needed in comparison to classical CNN 
architectures to obtain similar reconstruction results. 
We believe that the presented framework and the proposed network architectures can be useful for learning physical models in the context of imaging and 
inverse problems, especially, where a physical interpretation of 
the learned features is crucial to establish confidence in the 
imaging task.

\section*{Acknowledgements}
We thank Jonas Adler and Sebastian Lunz for valubale discussions and comments. We acknowledge the support of NVIDIA Corporation with one Titan Xp GPU.
\bibliographystyle{siam}
\bibliography{Inverse_problems_references_2018}

\end{document}